\documentclass{article}

\PassOptionsToPackage{numbers, compress}{natbib}

 \usepackage[final]{neurips_2022}

\usepackage[utf8]{inputenc} 
\usepackage[T1]{fontenc}    
\usepackage{hyperref}       
\usepackage{url}            
\usepackage{booktabs}       
\usepackage{amsfonts}       
\usepackage{nicefrac}       
\usepackage{microtype}      
\usepackage{xcolor}         

\usepackage{amsmath}
\usepackage{amssymb}
\usepackage{graphicx}
\usepackage{amsthm}
\usepackage{subfig}
\usepackage{float}
\usepackage{caption}
\usepackage{centernot}
\usepackage{wrapfig}

\usepackage[ruled]{algorithm2e}

\newcommand{\xhdr}[1]{\textbf{#1}}

\DeclareMathOperator*{\argmin}{arg\,min}
\newcommand{\E}{\mathop{\mathbb{E}}}
\newcommand{\R}{\mathbb{R}}
\newcommand{\X}{\mathbf{x}}
\newcommand{\grad}{\nabla_{\xi}}

\newcommand{\op}{\oplus}

\newcommand{\LIME}{\texttt{LIME}\xspace}
\newcommand{\CLIME}{\texttt{C-LIME}\xspace}
\newcommand{\SHAP}{\texttt{SHAP}\xspace}
\newcommand{\KSHAP}{\texttt{KernelSHAP}\xspace}
\newcommand{\Occlusion}{\texttt{Occlusion}\xspace}
\newcommand{\Grads}{\texttt{Vanilla Gradients}\xspace}
\newcommand{\SmoothGrad}{\texttt{SmoothGrad}\xspace}
\newcommand{\IntGrad}{\texttt{Integrated Gradients}\xspace}
\newcommand{\GradtimesInput}{\texttt{Gradient x Input}\xspace}

\newtheorem{prop}{Proposition}
\newtheorem{defn}{Definition}
\newtheorem{thm}{Theorem}
\newtheorem*{thm*}{Theorem}

\newtheorem{remark}{Remark}

\title{Which Explanation Should I Choose? \\A Function Approximation Perspective to Characterizing Post Hoc Explanations}

\author{%
  Tessa Han \\
  Harvard University\\
  Cambridge, MA \\
  \texttt{than@g.harvard.edu}
  \And
  Suraj Srinivas \\
  Harvard University\\
  Cambridge, MA \\
  \texttt{ssrinivas@seas.harvard.edu}
  \And
  Himabindu Lakkaraju \\
  Harvard University\\
  Cambridge, MA \\
  \texttt{hlakkaraju@hbs.edu}
}

\begin{document}

\maketitle

\vspace{-0.1cm}
\begin{abstract}
A critical problem in the field of post hoc explainability is the lack of a common foundational goal among methods. For example, some methods are motivated by function approximation, some by game theoretic notions, and some by obtaining clean visualizations. This fragmentation of goals causes not only an inconsistent conceptual understanding of explanations but also the practical challenge of not knowing which method to use when.

In this work, we begin to address these challenges by unifying eight popular post hoc explanation methods (LIME, C-LIME, KernelSHAP, Occlusion, Vanilla Gradients, Gradients $\times$ Input, SmoothGrad, and Integrated Gradients). We show that these methods all perform local function approximation of the black-box model, differing only in the neighbourhood and loss function used to perform the approximation. This unification enables us to (1) state a \textit{no free lunch theorem for explanation methods}, demonstrating that no method can perform optimally across all neighbourhoods, and (2) provide a \textit{guiding principle} to choose among methods based on faithfulness to the black-box model. We empirically validate these theoretical results using various real-world datasets, model classes, and prediction tasks. 

By bringing diverse explanation methods into a common framework, this work (1)~advances the conceptual understanding of these methods, revealing their shared local function approximation objective, properties, and relation to one another, and (2) guides the use of these methods in practice, providing a principled approach to choose among methods and paving the way for the creation of new ones.


\end{abstract}

\vspace{-0.2cm}
\section{Introduction} \label{secn:intro}
\vspace{-0.1cm}

As machine learning models become increasingly complex and are increasingly deployed in high-stakes settings (e.g., medicine \cite{yu2018artificial}, law \cite{walters2021artificial}, and finance \cite{cao2022ai}), there is a growing emphasis on understanding how models make predictions so that decision-makers (e.g., doctors, judges, and loan officers) can assess the extent to which they can trust model predictions. To this end, several post hoc explanation methods have been developed, including \LIME~\cite{ribeiro2016lime}, \CLIME \cite{agarwal2021clime}, \SHAP~\cite{lundberg2017shap}, \Occlusion~\cite{zeiler2013occlusion}, \Grads~\cite{simonyan2014grad}, \GradtimesInput~\cite{shrikumar2017learning}, \SmoothGrad~\cite{smilkov2017smoothgrad}, and \IntGrad~\cite{sundararajan2017integratedgrad}. 
However, different methods have different goals. Such differences lead to both conceptual and practical challenges to understanding and using explanation methods, thwarting progress in the field.

From a conceptual standpoint, the misalignment of goals among methods leads to an inconsistent view of explanations. What is an explanation? This is unclear as different methods have different notions of explanation. Depending on the method, explanations may be local function approximations (\LIME and \CLIME), Shapley values (\SHAP), raw gradients (\Grads), raw gradients scaled by the input (\GradtimesInput), de-noised gradients (\SmoothGrad), or a straight-line path integral of gradients (\IntGrad). Furthermore, the lack of a common mathematical framework for studying these diverse methods prevents a systematic understanding of these methods and their properties. To address these challenges, this paper unifies diverse explanation methods under a common framework, showing that diverse methods share a common motivation of local function approximation, and uses the framework to investigate and evaluate properties of these methods.

From a practical standpoint, the misalignment of goals among methods leads to the disagreement problem \cite{krishna2022disagreement}, the phenomenon that different methods provide disagreeing explanations for the same model prediction. Not only do different methods often generate disagreeing explanations in practice, but practitioners do not have a principled approach to select among explanations, resorting to ad hoc heuristics such as personal preference \cite{krishna2022disagreement}. These findings prompt one to ask why explanation methods disagree and how to select among them in a principled manner. This paper addresses these questions, providing both an explanation for the disagreement problem and a principled approach to select among methods.

Thus, to address these conceptual and practical challenges, we study post hoc explanation methods from a function approximation perspective. We formalize a mathematical framework that unifies and characterizes diverse methods and that provides a principled approach to select among methods. Our work makes the following contributions:

\begin{enumerate}
    \item We show that eight diverse, popular explanation methods (\LIME, \CLIME, \KSHAP, \Occlusion, \Grads, \GradtimesInput, \SmoothGrad, and \IntGrad) all perform local function approximation of the black-box model, differing only in the neighbourhoods and loss functions used to perform the approximation.
    
    \item We introduce a \textit{no free lunch theorem for explanation methods} which demonstrates that no single explanation method can perform local function approximation faithfully across all neighbourhoods, which in turn calls for a principled approach to select among methods.
    
    \item To select among methods, we set forth a \textit{guiding principle} based on function approximation, deeming a method to be effective if its explanation recovers the black-box model when the two are in the same model class (i.e., if the explanation perfectly approximates the black-box model when possible).
    
    \item We empirically validate the theoretical results above using various real-world datasets, model classes, and prediction tasks.
\end{enumerate}

\section{Related Work}

\xhdr{Post hoc explanation methods.}
Post hoc explanation methods can be classified based on model access (black-box model vs. access to model internals), explanation scope (global vs. local), search technique (perturbation-based vs. gradient-based), and basic unit of explanation (feature importance vs. rule-based). This paper focuses on local post hoc explanation methods based on feature importance. It analyzes four perturbation-based methods (\LIME, \CLIME, \KSHAP, and \Occlusion) and four gradient-based methods (\Grads, \GradtimesInput, \SmoothGrad, and \IntGrad).

\xhdr{Connections among post hoc explanation methods.}
Prior works have taken initial steps towards characterizing post hoc explanation methods and the connections among them. \citet{agarwal2021clime} proved that \CLIME and \SmoothGrad converge to the same explanation in expectation. \citet{lundberg2017shap} proposed a framework based on Shapley values to unify binary perturbation-based explanations. \citet{covert2021removal} found that many perturbation-based methods share the property of estimating feature importance based on the change in model behavior upon feature removal. In addition, \citet{ancona2018gradient} analyzed four gradient-based explanation methods and the conditions under which they produce similar explanations. However, these analyses are based on mechanistic properties of methods (e.g., Shapley values or feature removal), are limited in scope (connecting only two methods, only perturbation-based methods, or only gradient-based methods), and do not inform when one method is preferable to another. In contrast, this paper formalizes a mathematical framework based on the concept of local function approximation, unifies eight diverse methods (spanning perturbation-based and gradient-based methods), and guides the use of these methods in practice.

\xhdr{Properties of post hoc explanation methods.}
Prior works have examined various properties of post hoc explanation methods, including faithfulness to the black-box model \cite{adebayo2018sanity, srinivas2018gradmatch, hooker2019benchmark}, robustness to adversarial attack \cite{ghorbani2019interpretation, slack2019can, dombrowski2019explanations, adebayo2018sanity, alvarez2018robustness}, and fairness across subgroups \cite{dai2022fairness}. This paper focuses on explanation faithfulness. Related works \cite{adebayo2018sanity, srinivas2018gradmatch, hooker2019benchmark} assessed explanations generated by gradient-based methods, finding that they are not always faithful to the underlying model. Different from these works, this paper provides a framework for generating faithful explanations in the first place, theoretically characterizes the faithfulness of existing methods in different input domains, and provides a principled approach to select among methods and develop new ones based on explanation faithfulness.

\section{Explanation as Local Function Approximation}\label{sec:LFA}

In this section, we formalize the local function approximation framework and show its connection to existing explanation methods. We start by defining the notation used in the paper.

\xhdr{Notation.} Let $f: \mathcal{X} \rightarrow \mathcal{Y}$ be the black-box function we seek to explain in a post hoc manner, with input domain $\mathcal{X}$ (e.g., $\mathcal{X} = \R^d$ or $\{0,1\}^d$) and output domain $\mathcal{Y}$ (e.g., $\mathcal{Y} = \R$ or $ [0,1]$). Let $\mathcal{G} = \{g: \mathcal{X} \rightarrow \mathcal{Y} \}$ be the class of interpretable models used to generate a local explanation for $f$ by selecting a suitable interpretable model $g \in \mathcal{G}$. 

We characterize locality around a point $\X_0 \in \mathcal{X}$ using a noise random variable $\xi$ which is sampled from distribution $\mathcal{Z}$. Let $\X_\xi = \X_0 \op \xi$ be a perturbation of $\X_0$ generated by combining $\X_0$ and $\xi$ using a binary operator $\op$ (e.g., addition, multiplication). Lastly, let $\ell(f,g,\X_0, \xi) \in \R^+$ be the loss function (e.g., squared error, cross-entropy) measuring the distance between $f$ and $g$ over the noise random variable $\xi$ around $\X_0$. 

We now define the local function approximation framework.

\begin{defn} \label{defn:lfa}
\textbf{Local function approximation} (LFA) of a black-box model $f$ on a neighbourhood distribution $\mathcal{Z}$  around $\X_0$ by an interpretable model class $\mathcal{G}$ and a loss function $\ell$ is given by 

\vspace{-0.4cm}
\begin{align}
\label{eqn:LFA}
g^{*} = \argmin_{g \in \mathcal{G}} \E_{\xi \sim \mathcal{Z}} \ell(f, g, \X_0, \xi)
\end{align}

where a valid loss $\ell$ is such that $\E_{\xi \sim \mathcal{Z}} \ell(f,g,\X_0,\xi) = 0 \iff f(\X_\xi) = g(\X_\xi)~~~\forall \xi \sim \mathcal{Z}$     

\end{defn}

The LFA framework is a formalization of the function approximation perspective first introduced by \LIME \cite{ribeiro2016lime} to motivate local explanations. Note that this conceptual framework is distinct from the algorithm introduced by \LIME. We elaborate on this distinction below.

(1) The LFA framework requires that $f$ and $g$ share the same input domain $\mathcal{X}$ and output domain $\mathcal{Y}$, a fundamental prerequisite for function approximation. This implies, for example, that using an interpretable model g with binary inputs ($\mathcal{X} = \{0,1\}^d$) to approximate a black-box model $f$ with continuous inputs ($\mathcal{X} = \R^d$), as proposed in \LIME, is not true function approximation. 

(2) By imposing a condition on the loss function, the LFA framework ensures model recovery under specific conditions: $g^*$ recovers $f$ (i.e., $g^* = f$) through LFA when $f$ itself is of the interpretable model class $\mathcal{G}$ (i.e., $f \in \mathcal{G}$) and perturbations span the input domain of $f$ (i.e., $\text{domain}(\X) = \mathcal{X}$). This is a key distinction between the LFA framework and \LIME (which has no such requirement) and guides the characterization of explanation methods in Section \S \ref{sec:whichmethod}.

(3) Efficiently minimizing Equation \ref{eqn:LFA} requires following standard machine learning methodology of splitting the perturbation data into train / validation / test sets and tuning hyper-parameters on the validation set to ensure generalization. To our knowledge, implementations of \LIME do not adopt this procedure, making it possible to overfit to a small number of perturbations.

The LFA framework is generic enough to accommodate a variety of explanation methods. In fact, we show that specific instances of this framework converge to existing methods, as summarized in Table~\ref{table:meta-algo-instances}. At a high level, existing methods use a linear model $g$ to locally approximate the black-box model $f$ in different input domains (binary or continuous) over different local neighbourhoods specified by noise random variable $\xi$ (where $\xi$ is binary or continuous, drawn from a specified distribution, and combined additively or multiplicatively with point $\X_0$) using different loss functions (squared-error or gradient-matching loss). We discuss the details of these connections in the following sections.

\begin{table}[]
    \centering
    \begin{tabular}{c|c|c}
        \textbf{Explanation Method} & \textbf{Local Neighbourhood $\mathcal{Z}$ around $\X_0$} & \textbf{Loss Function $\ell$}\\
        \midrule
         C-LIME & $\X_0 + \xi; ~\xi (\in \mathbb{R}^d) \sim \text{Normal}(0, \sigma^2)$ & Squared Error \\
         SmoothGrad & $\X_0 + \xi;~ \xi (\in \mathbb{R}^d) \sim \text{Normal}(0, \sigma^2)$ & Gradient Matching  \\
         Vanilla Gradients & $\X_0 + \xi;~ \xi (\in \mathbb{R}^d) \sim \text{Normal}(0, \sigma^2), \sigma \rightarrow 0$ & Gradient Matching\\
         \midrule
         Integrated Gradients & $\xi \X_0; ~\xi (\in \mathbb{R}) \sim \text{Uniform}(0,1)$ & Gradient Matching \\
         Gradients $\times$ Input & $\xi \X_0; ~\xi (\in \mathbb{R}) \sim \text{Uniform}(a,1), a \rightarrow 1 $ & Gradient Matching\\
         \midrule 
         LIME & $\X_0 \odot \xi; ~\xi (\in \{0, 1\}^d) \sim \text{Exponential kernel}$ & Squared Error\\
         KernelSHAP & $\X_0 \odot \xi; ~\xi (\in \{0, 1\}^d) \sim \text{Shapley kernel}$ & Squared Error \\
         Occlusion & $\X_0 \odot \xi; ~\xi (\in \{0, 1\}^d) \sim \text{Random one-hot vectors}$ & Squared Error\\
    \end{tabular}
    \caption{Correspondence of existing explanation methods to instances of the LFA framework. Existing methods perform LFA of a black-box model $f$ using the interpretable model class $\mathcal{G}$ of linear models where $g(\X) = \mathbf{w}^\top\X$ over a local neighbourhood $\mathcal{Z}$ around point $\X_0$ based on a loss function~$\ell$. Exponential and Shapley kernels are defined in Appendix~\ref{app:proofs-all}.}
    \label{table:meta-algo-instances}
    \vspace{-6mm}
\end{table}

%
\subsection{LFA with Continuous Noise: Gradient-Based Explanation Methods}\label{sec:gradients}

To connect gradient-based explanation methods to the LFA framework, we leverage the gradient-matching loss function $\ell_{gm}$. We define $\ell_{gm}$ and show that it is a valid loss function for LFA.

\vspace{-0.4cm}
\begin{align} 
\ell_{gm}(f, g, \X_0, \xi) = \| \nabla_{\xi} f(\X_0 \oplus \xi) - \nabla_{\xi} g(\X_0 \oplus \xi) \|_2^2
\end{align}

This loss function has been previously used in the contexts of generative modeling (where it is dubbed score-matching) \cite{JMLR:v6:hyvarinen05a} and model distillation \cite{srinivas2018gradmatch}. However, to our knowledge, its use in interpretability is novel.

\begin{prop} \label{prop:gm-loss}
The gradient-matching loss function $\ell_{gm}$ is a valid loss function for LFA up to
a constant, i.e., $\E_{\xi \sim \mathcal{Z}} \ell_{gm}(f,g,\X_0,\xi) = 0 \iff f(\X_\xi) = g(\X_\xi) + C~~~\forall \xi \sim \mathcal{Z}$, where $C \in \R$.    
\end{prop}

\begin{proof} If $f(\X_\xi) = g(\X_\xi)$, then $\grad f(\X_\xi) = \grad g(\X_\xi)$ and it follows from the definition of $\ell_{gm}$ that $\ell_{gm}=0$. Integrating $\grad f(\X_\xi) = \grad g(\X_\xi)$ gives $f(\X_\xi) = g(\X_\xi) + C$.
\end{proof}

Proposition~\ref{prop:gm-loss} implies that, when using the linear model class $\mathcal{G}$ parameterized by $g(\X) = \mathbf{w}^\top\X + b$ to approximate $f$, $g^*$ recovers $\mathbf{w}$ but not $b$. This can be fixed by setting $b = f(0)$.

\begin{thm}
LFA with gradient-matching loss is equivalent to (1) \SmoothGrad for additive continuous Gaussian noise, which
converges to \Grads in the limit of a small standard deviation for the Gaussian distribution; and
(2) \IntGrad for multiplicative continuous Uniform noise, which converges to \GradtimesInput
in the limit of a small support for the Uniform distribution.
\end{thm}

\textit{Proof Sketch.} For \SmoothGrad and \IntGrad, the idea is that these methods are exactly the first-order stationary points of the gradient-matching loss function under their respective noise distributions. In other words, the weights of the interpretable model $g$ that minimize the loss function is the explanation returned by each method. For \Grads and \GradtimesInput, the result is derived by taking the specified limits and using the Dirac delta function to calculate the limit. In the limit, the weights of the interpretable model $g$ converge to the explanation of each method. The full proof is in Appendix~\ref{app:proofs-all}.

Along with gradient-based methods, \CLIME (a perturbation-based method) is an instance of the LFA framework by definition, using the squared-error loss function. The analysis in this section characterizes methods that use continuous noise. It does not extend to binary or discrete noise methods because gradients and continuous random variables do not apply in these domains. In the next section, we discuss binary noise methods.

\subsection{LFA with Binary Noise: LIME, KernelSHAP and Occlusion maps}

\begin{thm}
LFA with multiplicative binary noise and squared-error loss is equivalent to (1) LIME for noise sampled from an unnormalized exponential kernel over binary vectors; (2) KernelSHAP for noise sampled from an unnormalized Shapley kernel; and (3) Occlusion for noise in the form of one-hot vectors.
\end{thm}

\textit{Proof Sketch.} 
For \LIME and \KSHAP, the equivalence is mostly by definition: these methods have components that correspond to the interpretable model $g$ and the loss function $\ell$ of the LFA framework and we need only to determine the local neighbourhood $\mathcal{Z}$. We define the local neighbourhood $\mathcal{Z}$ using each method's weighting kernel. In this setup, the LFA framework yields the respective explanation methods in expectation via importance sampling. For \Occlusion, the equivalence involves enumerating all perturbations, specifying an appropriate loss function, and computing the resulting stationary points of the loss function. The full proof is in Appendix~\ref{app:proofs-all}.


\subsection{Which Methods Do Not Perform LFA?}

Some popular explanation methods are not instances of the LFA framework due to their properties. These methods include guided backpropagation \cite{springenberg2015striving}, DeconvNet \cite{noh2015learning}, Grad-CAM \cite{selvaraju2017grad}, Grad-CAM++ \cite{chattopadhay2018grad}, FullGrad \cite{srinivas2019full}, and DeepLIFT \cite{shrikumar2017learning}. Further details are in Appendix~\ref{app:notLFA}.
\section{When Do Explanations Perform Model Recovery?}\label{sec:whichmethod}

Having described the LFA framework and its connections to existing explanation methods, we now leverage this framework to analyze the performance of methods under different conditions. We introduce a \textit{no free lunch theorem for explanation methods}, inspired by classical no free lunch theorems in learning theory and optimization. Then, we assess the ability of existing methods to perform \textit{model recovery} based on which we provide recommendations for choosing among methods.

\subsection{No Free Lunch Theorem for Explanation Methods}\label{sub-sec:NFL}
An important implication of the function approximation perspective is that no explanation can be optimal across all neighbourhoods because each explanation is designed to perform LFA in a specific neighbourhood. This is especially true for explanations of non-linear models. We formalize this intuition into the following theorem.

\begin{thm}[No Free Lunch for Explanation Methods]
\label{thm:NFL}
Consider explaining a black-box model $f$ around point $\X_0$ using an interpretable model $g$ from model class $\mathcal{G}$ and a valid loss function $\ell$ where the distance between $f$ and $\mathcal{G}$ is given by $d(f, \mathcal{G}) = \min_{g \in \mathcal{G}} \max_{\X \in \mathcal{X}} \ell(f,g,0,\X)$.

Then, for any explanation $g^*$ over a neighbourhood distribution $\xi_1~\sim~ \mathcal{Z}_1$ such that $\max_{\xi_1} \ell(f,g^*,\X_0, \xi_1) \leq \epsilon$, there always exists another neighbourhood $\xi_2 \sim \mathcal{Z}_2$ such that $\max_{\xi_2} \ell(f,g^*,\X_0, \xi_2) \geq d(f, \mathcal{G})$.
\end{thm}

\textit{Proof Sketch.} The idea is that, given an explanation obtained by using $g$ to approximate $f$ over a specific local neighbourhood $\mathcal{Z}$, it is always possible to find a local neighbourhood over which this explanation does not perform well (i.e., does not perform faithful LFA). Thus, no single explanation method can perform well over all local neighbourhoods. The proof entails constructing an ``adversarial'' input for an explanation $g^*$ such that $g^*$ has a large loss for this input and then creating a 
neighbourhood that contains this adversarial input which will provably have a large loss. The magnitude of this loss is $d(f, \mathcal{G})$, the distance between $f$ and the model class $\mathcal{G}$, inspired by the Haussdorf distance. The proof is generic and makes no assumptions regarding the forms of $\ell$, $\mathcal{G}$ or $\mathcal{Z}_1$. The full proof is in Appendix~\ref{app:nfl-theorem}.

Thus, an explanation on a finite $\mathcal{Z}_1$ necessarily cannot approximate function behaviour at all other points, especially when $\mathcal{G}$ is less expressive than $f$, which is indicated by a large value of $d(f, \mathcal{G})$. Thus, in the general case, one cannot perform model recovery as $\mathcal{G}$ is less expressive than $f$.

An important implication of Theorem~\ref{thm:NFL} is that seeking to find the ``best'' explanation without specifying a corresponding neighbourhood is futile as no universal ``best'' explanation exists. Furthermore, once the neighbourhood is specified, the best explanation is exactly the one given by the corresponding instance of the LFA framework. 

In the next section, we consider the special case when $d(f, \mathcal{G}) = 0$ (i.e., when $f \in \mathcal{G}$), where Theorem~\ref{thm:NFL} does not apply because the same explanation can be optimal for multiple neighbourhoods and model recovery is thus possible.

\subsection{Characterizing Explanation Methods via Model Recovery}\label{sub-secn:choose-N}

Next, we formally state the model recovery condition for explanation methods. Then, we use this condition as a guiding principle to choose among methods.

\begin{defn}[Model Recovery: Guiding Principle]
Given an instance of the LFA framework with a black-box model $f$ such that $f \in \mathcal{G}$ and a specific noise type (e.g., Gaussian, Uniform), an explanation method performs model recovery if there exists some noise distribution $\mathcal{Z}$ such that LFA returns $g^* = f$.
\end{defn}

In other words, when the black-box model $f$
itself is of the interpretable model class $G$, there must exist some setting of the noise distribution (within the noise type specified in the instance of the LFA framework) that is able to recover the black-box model. Thus, in this special case, we require \textit{local function approximation} to lead to \textit{global model recovery} over all inputs. This criterion can be thought of as a ``sanity check'' for explanation methods to ensure that they remain faithful to the black-box model.

Next, we analyze the impact of the choice of perturbation neighbourhood $\mathcal{Z}$, the binary operator $\oplus$, and the interpretable model class $\mathcal{G}$ on an explanation method's ability to satisfy the model recovery guiding principle in different input domains $\mathcal{X}$. Note that while we can choose $\mathcal{Z}$, $\oplus$, and $\mathcal{G}$, we cannot choose $\mathcal{X}$, the input domain.

\textbf{Which explanation should I choose for continuous $\mathcal{X}$?}
We now analyze the model recovery properties 
of existing explanation methods when the input domain is continuous. We consider methods based on additive continuous noise (\SmoothGrad, \Grads, and \CLIME), multiplicative continuous noise (\IntGrad and \GradtimesInput), and multiplicative binary noise (\LIME, \KSHAP, and \Occlusion). For these methods, we make the following remark regarding model recovery for the 
class of linear models.

\vspace{-0.1cm}
\begin{remark} \label{rmk:recovery-contX}
For $\mathcal{X} = \R^d$ and linear models $f$ and $g$ where $f(\X) = \mathbf{w}_f^\top \X$ and $g(\X) = \mathbf{w}_g^\top \X$, additive continuous noise methods recover f (i.e., $\textbf{w}_g = \textbf{w}_f$) while multiplicative continuous and multiplicative binary noise methods do not and instead recover $\textbf{w}_g = \textbf{w}_f \odot \textbf{x}$.
\end{remark}

This remark can be verified by directly evaluating the explanations (weights) of linear models, where the gradient exactly corresponds to the weights.

Note that the inability of multiplicative continuous noise methods to recover the black-box model is not due to the multiplicative nature of the noise, but due to the parameterization of the loss function. Specifically, these methods (implicitly) use the loss function $\ell(f,g,\X_0,\xi) = \| \grad f(\X_\xi) - \grad g(\xi) \|_2^2$. Slightly changing the loss function to $\ell(f,g,\X_0, \xi) = \| \grad f(\X_\xi) - \grad g(\X_\xi) \|_2^2$, i.e., replacing $g(\xi)$ with $g(\X_\xi)$, would enable $g^*$ to recover $f$. This would change \IntGrad to $\int_{\alpha=0}^1 \nabla_{\alpha \X} f(\alpha \X)$ (omitting the input multiplication term) and \GradtimesInput to \Grads.

A similar argument can be made for binary noise methods which parameterize the loss function as $\ell(f,g,\X_0, \xi) = \| f(\X_\xi) - g(\xi) \|^2$. By changing the loss function to $\ell(f,g,\X_0, \xi) = \| f(\X_\xi) - g(\X_\xi) \|^2$, binary noise methods can recover $f$ for the case described in Remark~\ref{rmk:recovery-contX}. However, binary noise methods for continuous domains are unreliable, as there are cases where, despite the modification to $\ell$, model recovery is not guaranteed. The following is an example of this scenario.

\begin{remark} \label{rmk:sinusoid}
For $\mathcal{X} = \R^d$, periodic functions $f$ and $g$ where $f(\X) =  \sum_{i=1}^{d} \sin(\mathbf{w}_{f_i} \odot \X_{i})$ 
and $g(\X) = \sum_{i=1}^{d} \sin(\mathbf{w}_{g_i} \odot \X_{i})$, and an integer $n$, binary noise methods do not perform model recovery for $|w_{f_i}| \geq \frac{n \pi}{\X_{0_i}}$.
\end{remark}

This is because, for the conditions specified, $\sin(\mathbf{w}_{f_i} \X_{0_i}) = \sin(\pm n \pi) = \sin(0) = 0$, i.e., $\sin(\mathbf{w}_{f_i} \X_{0_i})$ outputs zero for all binary perturbations, thereby preventing model recovery. In this case, the discrete nature of the noise makes model recovery impossible. In general, discrete noise is inadequate for the recovery of models with large frequency components. 

\textbf{Which explanation should I choose for binary $\mathcal{X}$?} In the binary domain, continuous noise methods are invalid, restricting the choice of methods to binary noise methods. For reasons discussed above, methods with perturbation neighbourhoods characterized by multiplicative binary perturbations (e.g., \LIME, \KSHAP, and \Occlusion) only enable $g^*$ to recover $f$ in the binary domain. Note that the sinusoidal example in Remark~\ref{rmk:sinusoid} does not apply in this regime due to the continuous nature of its domain.

\textbf{Which explanation should I choose for discrete $\mathcal{X}$?} In the discrete domain, continuous noise methods are also invalid. In addition, binary noise methods (e.g., \LIME, \KSHAP and \Occlusion) cannot be used either because model recovery is not guaranteed in the sinusoidal case (Remark~\ref{rmk:sinusoid}), following similar logic to that presented for continuous noise. Note that none of the existing methods in Table~\ref{table:meta-algo-instances} perform general discrete perturbations, suggesting that these methods are not suitable for the discrete domain. Thus, in the discrete domain, a user can apply the LFA framework to define a new explanation method, specifying an appropriate discrete noise type. In the next section, we discuss more broadly about how one can use the LFA framework to create novel explanation methods.

\subsection{Designing Novel Explanations with LFA}

The LFA framework not only unifies existing explanation methods but also guides the creation of new ones. To explain a given black-box model prediction using the LFA framework, a user must specify the (1) interpretable model class $\mathcal{G}$, (2) neighbourhood distribution $\mathcal{Z}$, (3) loss function $\ell$, and (4) binary operator $\oplus$ to combine the input and the noise. Specifying these four components completely specifies an instance of the LFA framework, thereby generating an explanation method tailored to a given context. 

To illustrate this, consider a scenario in which a user seeks to create a sparse variant of \SmoothGrad that yields non-zero gradients for only a small number of features (``\texttt{SparseSmoothGrad}''). Designing \texttt{SparseSmoothGrad} only requires the addition of a regularization term to the loss function used in the \SmoothGrad instance of the LFA framework (e.g., $\ell = \ell_{SmoothGrad} + \| \grad g(\X_\xi) \|_0$), 
at which point, sparse solvers may be employed to solve the problem. Note that, unlike \SmoothGrad, \texttt{SparseSmoothGrad} does not have a closed form solution, but that is not an issue for the LFA framework. More generally, by allowing customization of (1), (2), (3), and (4), the LFA framework creates new explanation methods through ``variations on a theme''.

We summarize Section~\S \ref{sec:whichmethod} as a table in Appendix \ref{app:summary-model-recovery} and discuss the practical implications of Section~\S \ref{sec:whichmethod} by providing the following recommendation for choosing among explanation methods.

\textbf{Recommendation for choosing among explanation methods.} 
In general, choose methods that satisfy the guiding principle of model recovery in the input domain in question.
For continuous data, use additive continuous noise methods (e.g., \SmoothGrad, \Grads, \CLIME) or modified multiplicative continuous noise methods (e.g., \IntGrad, \GradtimesInput) as described in Section~\S\ref{sub-secn:choose-N}. 
For binary data, use binary noise methods (e.g., \LIME, \KSHAP, \Occlusion). Given that methods that use discrete noise do not exist, in case of discrete data, design novel explanation methods using the LFA framework with discrete noise neighbourhoods. Within each input domain, choosing among appropriate methods boils down to determining the perturbation neighbourhood most suitable in the given context.

\section{Empirical Evaluation}\label{sec:expts}

In this section, we present an empirical evaluation of the LFA framework. We first describe the experimental setup and then discuss three experiments and their findings. 

\subsection{Datasets, Models, and Metrics}

\textbf{Datasets.}
We experiment with two real-world datasets for two prediction tasks. The first dataset is
the life expectancy dataset from the World Health Organization (WHO) \cite{dataset2018who}. It consists of countries’ demographic, economic, and health factors from 2000 to 2015, with 2,938 observations for 20 continuous features. We
use this dataset to perform regression, predicting life expectancy. The other dataset is the home equity line of credit (HELOC) dataset from FICO \cite{dataset2019heloc}. It consists of information on HELOC applications, with 9,871 observations for 24 continuous features. We use this dataset to perform classification, predicting whether an applicant made payments without being 90 days overdue. Additional dataset details are described in Appendix~\ref{app:exp-info}.

\textbf{Models.}
For each dataset, we train four models: a simple model (linear regression for the WHO dataset and logistic regression for the HELOC dataset) that can satisfy conditions of the guiding principle and three more complex models (neural networks of varying complexity) that are more reflective of real-world applications. Model architectures and performance are described in Appendix~\ref{app:exp-info}.

\textbf{Metrics.}
To measure the similarity between two vectors (e.g., between two sets of explanations or between an explanation and the true model weights), we use L1 distance and cosine distance. L1 distance ranges between [0, $\infty$) and is 0 when two vectors are the same. Cosine distance ranges between [0, 2] and is 0 when the angle between two vectors is $0^\circ$ (or $360^\circ$). For both metrics, the lower the value, the more similar two given vectors are.


\subsection{Experiments}

Here, we describe the setup of the experiments, present results, and discuss their implications.


\begin{figure}
    \centering
    \includegraphics[width=0.90\textwidth]{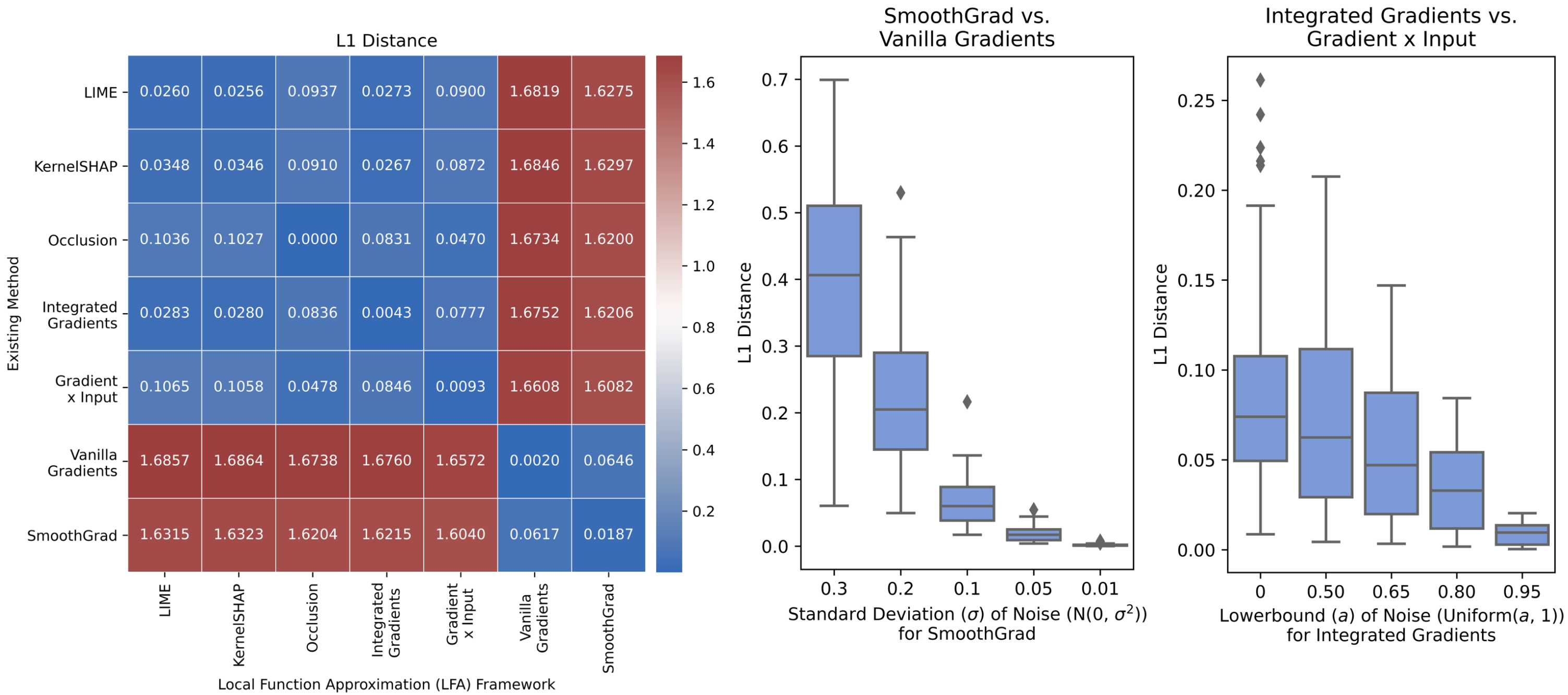}
    \vspace*{-4mm}
    \caption*{\raggedright \hspace{3.5cm} (a) \hspace{4.1cm} (b) \hspace{2.8cm}   (c)}
    \vspace*{-1mm}
    \caption{Correspondence between existing explanation methods and instances of the LFA framework. (a) Heatmap of average L1 distance between pairs of explanations. Boxplots of L1 distance between explanations of (b) \SmoothGrad and \Grads and (c) \IntGrad and \GradtimesInput. The lower the L1 distance, the more similar two explanations are. Results indicate that existing explanation methods are instances of the LFA framework.}
    \label{fig:exp1-lfa}
\end{figure}

\textbf{Experiment 1: Existing explanation methods are instances of the LFA framework.} First, we compare existing methods with corresponding instances of the LFA framework to assess whether they generate the same explanations. To this end, we use seven methods to explain the predictions of black-box models for 100 randomly-selected test set points. For each method, explanations are computed using either the existing method (implemented by Meta's Captum library \cite{kokhlikyan2020captum}) or the corresponding instance of the LFA framework (Table~\ref{table:meta-algo-instances}). The similarity of a given pair of explanations is measured using L1 distance and cosine distance.

The L1 distance values for a neural network with three hidden layers trained on the WHO dataset are shown in Figure~\ref{fig:exp1-lfa}. In Figure~\ref{fig:exp1-lfa}a, lowest L1 distance values appear in the diagonal of the heatmap, indicating that explanations generated by existing methods and corresponding instances of the LFA framework are very similar. Figures~\ref{fig:exp1-lfa}b and \ref{fig:exp1-lfa}c show that explanations generated by instances of the LFA framework corresponding to \SmoothGrad and \IntGrad converge to those of \Grads and \GradtimesInput, respectively. Together, these results demonstrate that, consistent with the theoretical results derived in Section~\S\ref{sec:LFA}, existing methods are instances of the LFA framework. In addition, the clustering of the methods in Figure~\ref{fig:exp1-lfa}a indicates that, consistent with the theoretical analysis in Section~\S\ref{sec:whichmethod}, for continuous data, \SmoothGrad and \Grads generate similar explanations while \LIME, \KSHAP, \Occlusion, \IntGrad, and \GradtimesInput generate similar explanations. We observe similar results across various datasets, models, and metrics (Appendix~\ref{app:exp1-lfa-all}).


\begin{figure}
    \centering
    \includegraphics[width=0.90\textwidth]{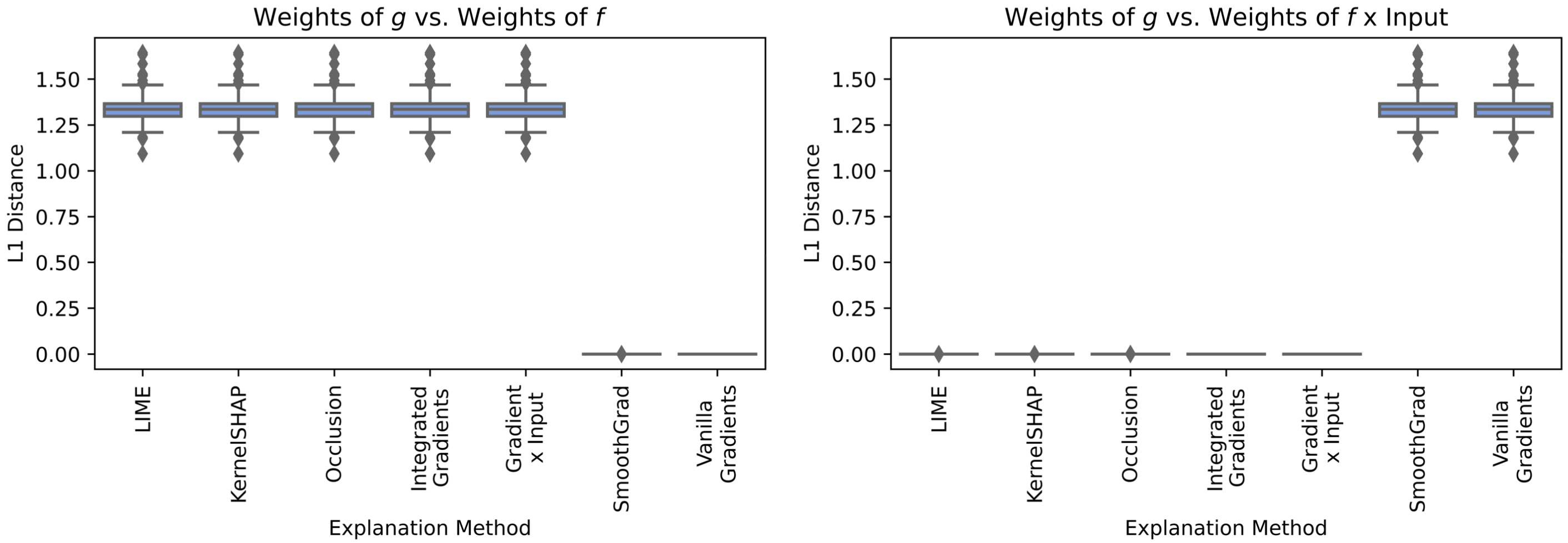}
    \vspace*{-4mm}
    \caption*{\raggedright \hspace{3.9cm} (a) \hspace{5.7cm} (b)}
    \caption{Analysis of model recovery. The lower the L1 distance, the more similar $g$'s weights are to (a) $f$'s weights or (b) $f$'s weights multiplied by the input. Results indicate that, for continuous data, additive continuous noise methods recover $f$'s weights, satisfying the guiding principle, while multiplicative binary and multiplicative continuous noise methods do not, recovering $f$'s weights multiplied by the input instead.}
    \label{fig:exp2-recovery}
\end{figure}

\textbf{Experiment 2: Some methods recover the underlying model while others do not (guiding principle).} Next, we empirically assess which existing methods satisfy the guiding principle, i.e., which methods recover the black-box model $f$ when $f$ is of the interpretable model class $\mathcal{G}$. We specify a setting in which $f$ and $g$ are of the same model class, generate explanations using each method, and assess whether $g$ recovers $f$ for each explanation. For the WHO dataset, we set $f$ and $g$ to be linear regression models and generate explanations for 100 randomly-selected test set points. Then, for each point, we compare $g$'s weights with $f$'s gradients alone or with $f$'s gradients multiplied by the input because, based on Section~\S\ref{sec:whichmethod}, some methods generate explanations on the scale of gradients while others on the scale of gradient-times-input. Note that, for linear regression, $f$'s gradients are $f$'s weights.

Results are shown in Figure~\ref{fig:exp2-recovery}. Consistent with Section~\S\ref{sec:whichmethod}, for continuous data, \SmoothGrad and \Grads recover the black-box model, thereby satisfying the guiding principle, while \LIME, \KSHAP, \Occlusion, \IntGrad, and \GradtimesInput do not. We observe similar results for the HELOC dataset using logistic regression models for $f$ and $g$ (Appendix~\ref{app:exp2-recovery-all}).


\begin{figure}
    \centering
    \includegraphics[width=0.90\textwidth]{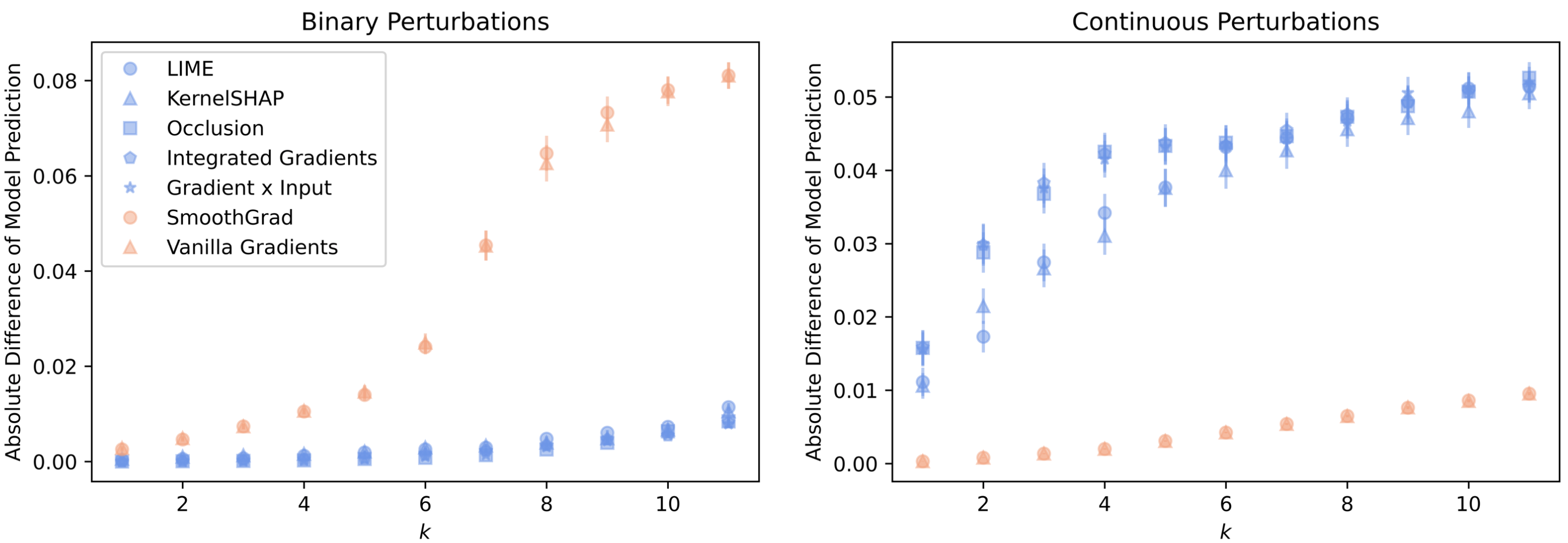}
    \vspace*{-4mm}
    \caption*{\raggedright \hspace{3.8cm} (a) \hspace{5.8cm} (b)}
    \vspace*{-1mm}
    \caption{Perturbation tests perturbing bottom $k$ features using (a) binary or (b) continuous noise. The lower the curve, the better a method identifies unimportant features. Results illustrate the no free lunch theorem, i.e., no single method performs best across all neighborhoods.}
    \label{fig:exp3-ptest}
\end{figure}

\textbf{Experiment 3: No single method performs best across all neighbourhoods (no free lunch theorem).} Lastly, we perform a set of experiments to illustrate the no free lunch theorem in Section~\S\ref{sec:whichmethod}. We generate explanations for black-box model predictions for 100 randomly-selected test set points and evaluate the explanations using perturbation tests based on top-$k$ or bottom-$k$ features. 
For perturbation tests based on top-$k$ features, the setup is as follows. For a given data point, $k$, and explanation, we identify the top-$k$ features and either replace them with zero (binary perturbation) or add Gaussian noise to them (continuous perturbation). Then, we calculate the absolute difference in model prediction before and after perturbation. For each point, we generate one binary perturbation (since such perturbations are deterministic) and 100 continuous perturbations (since such perturbations are random), computing the average absolute difference in model prediction for the latter. In this setup, methods that better identify important features yield larger changes in model prediction. For perturbation tests based on bottom-$k$ features, we follow the same procedure but perturb the bottom-$k$ features instead. In this setup, methods that better identify unimportant features yield smaller changes in model prediction. 

Results of perturbation tests based on bottom-$k$ features performed on explanations for a neural network with three hidden layers trained on the WHO dataset are displayed in Figure~\ref{fig:exp3-ptest}. Consistent with the no free lunch theorem in Section~\S\ref{sec:whichmethod}, 
\LIME, \KSHAP, \Occlusion, \IntGrad, and \GradtimesInput perform best on binary perturbation neighbourhoods (Figure~\ref{fig:exp3-ptest}a)
while \SmoothGrad and \Grads perform best on continuous perturbation neighborhoods (Figure~\ref{fig:exp3-ptest}b). We observe consistent results across perturbation test types (top-$k$ and bottom-$k$), datasets, and models (Appendix~\ref{app:exp3-p-test-all}). These results have important implications: one should carefully consider the perturbation neighborhood not only when selecting a method to generate explanations but also when selecting a method to evaluate explanations. In fact, the type of perturbations used to evaluate explanations directly determines explanation method performance.

\vspace{-0.2cm}
\section{Conclusions and Future Work} \label{secn:conclusion}
\vspace{-0.2cm}
In this work, we formalize the \emph{local function approximation (LFA)} framework and demonstrate that eight popular explanation methods can be characterized as instances of this framework with different local neighbourhoods and loss functions. We also introduce the \emph{no free lunch theorem for explanation methods}, showing that no single method can perform optimally across all neighbourhoods, and provide a \emph{guiding principle} for choosing among methods. 

The function approximation perspective captures the essence of an explanation -- a simplification of the real world (i.e., a black-box model) that is nonetheless accurate enough to be useful (i.e., predict outcomes of a set of perturbations). When the real world is ``simple'', an explanation should completely capture its behaviour, a hallmark expressed precisely by the guiding principle. When the requirements of two explanations are distinct (i.e., they are trained to predict different sets of perturbations), then the explanations are each accurate in their own domain and may disagree, a phenomenon captured by the no free lunch theorem.
 
Our work makes fundamental contributions. We \emph{unify} popular explanation methods, bringing diverse methods into a common framework. Unification brings \emph{conceptual coherence and clarity}: diverse explanation methods, even those seemingly unrelated to function approximation, perform LFA but differ in the way they perform it. Unification also enables \emph{theoretical simplicity}: to study diverse explanation methods, instead of analyzing each method individually, one can simply analyze the LFA framework and apply the findings to each method. An example of this is the no free lunch theorem which holds true for all instances of the LFA framework. Furthermore, our work provides \emph{practical guidance} by presenting a principled approach to select among methods and design new ones.

Our work also addresses key open questions in the field. In response to criticism about the lack of consensus in the field regarding the overarching goals of post hoc explainability \cite{lipton2018mythos}, our work points to function approximation as a principled goal. It also provides an explanation for the disagreement problem~\cite{krishna2022disagreement}, i.e., why different methods generate different explanations for the same model prediction. According to the LFA framework, this disagreement occurs because different methods approximate the black-box model over different neighbourhoods using different loss functions. 

Future research includes the following directions. First, we analyzed eight popular post hoc explanation methods and this analysis could be extended to other methods. Second, our work focuses on the faithfulness rather than interpretability of explanations. The latter is encapsulated in the ``interpretable'' model class $\mathcal{G}$, which includes all the information about human preferences with regards to interpretability. However, it is unclear what constitutes an interpretable explanation and elucidating this takes not only conceptual understanding but also human-computer interaction research such as user studies. These are important directions for future research.


\vspace{-0.215cm}
\section*{Acknowledgements}
\vspace{-0.215cm}
The authors would like to thank the anonymous reviewers for their helpful feedback and the following funding agencies for supporting this work. This work is supported in part by NSF awards $\#$IIS-2008461 and $\#$IIS-2040989, and research awards from Google, JP Morgan, Amazon, Harvard Data Science Initiative, and D\^{}3 Institute at Harvard. H.L. would like to thank Sujatha and Mohan Lakkaraju for their continued support and encouragement. T.H. is supported in part by an NSF GRFP fellowship. The views expressed here are those of the authors and do not reflect the official policy or position of the funding agencies. 


\bibliographystyle{unsrtnat}
\bibliography{references.bib}

\section*{Checklist}


\begin{enumerate}

\item For all authors...
\begin{enumerate}
  \item Do the main claims made in the abstract and introduction accurately reflect the paper's contributions and scope?
    \answerYes{See Abstract and Section \S\ref{secn:intro}.}
  \item Did you describe the limitations of your work?
    \answerYes{See Section \S\ref{secn:conclusion}.}
  \item Did you discuss any potential negative societal impacts of your work?
    \answerYes{See Section \S\ref{secn:conclusion}.}
  \item Have you read the ethics review guidelines and ensured that your paper conforms to them?
    \answerYes{}
\end{enumerate}

\item If you are including theoretical results...
\begin{enumerate}
  \item Did you state the full set of assumptions of all theoretical results?
    \answerYes{See Sections \S\ref{sec:LFA}, \S\ref{sec:whichmethod}, and Appendix.}
  \item Did you include complete proofs of all theoretical results?
    \answerYes{See Sections \S\ref{sec:LFA}, \S\ref{sec:whichmethod}, and Appendix.}
\end{enumerate}

\item If you ran experiments...
\begin{enumerate}
  \item Did you include the code, data, and instructions needed to reproduce the main experimental results (either in the supplemental material or as a URL)?
    \answerYes{We include a zip file with the code in the supplementary material. The code can also be found at the following repository: \url{https://github.com/AI4LIFE-GROUP/lfa}.}
  \item Did you specify all the training details (e.g., data splits, hyperparameters, how they were chosen)?
    \answerYes{See Appendix.}
  \item Did you report error bars (e.g., with respect to the random seed after running experiments multiple times)?
    \answerYes{See Figure~\ref{fig:exp3-ptest} and Appendix.}
  \item Did you include the total amount of compute and the type of resources used (e.g., type of GPUs, internal cluster, or cloud provider)?
    \answerYes{See Appendix.}
\end{enumerate}

\item If you are using existing assets (e.g., code, data, models) or curating/releasing new assets...
\begin{enumerate}
  \item If your work uses existing assets, did you cite the creators?
    \answerYes{See Section~\S\ref{sec:expts}.}
  \item Did you mention the license of the assets?
    \answerNA{}
  \item Did you include any new assets either in the supplemental material or as a URL?
    \answerNA{}
  \item Did you discuss whether and how consent was obtained from people whose data you're using/curating?
    \answerNA{We do not directly obtain data from individuals.}
  \item Did you discuss whether the data you are using/curating contains personally identifiable information or offensive content?
    \answerYes{To our knowledge, the data contains no such information nor content. See Appendix.}
\end{enumerate}

\item If you used crowdsourcing or conducted research with human subjects...
\begin{enumerate}
  \item Did you include the full text of instructions given to participants and screenshots, if applicable?
    \answerNA{We did not use crowdsourcing nor did we conduct research with human subjects.}
  \item Did you describe any potential participant risks, with links to Institutional Review Board (IRB) approvals, if applicable?
    \answerNA{We did not use crowdsourcing nor did we conduct research with human subjects.}
  \item Did you include the estimated hourly wage paid to participants and the total amount spent on participant compensation?
    \answerNA{We did not use crowdsourcing nor did we conduct research with human subjects.}
\end{enumerate}

\end{enumerate}


\newpage
\appendix

\section{Appendix}

\subsection{Proofs: Existing Methods are Instances of the LFA Framework (Section~\ref{sec:LFA})}
\label{app:proofs-all}

\subsubsection{LIME}\label{app:proof-lime}

The instance of the LFA framework with (1) interpretable model class $\mathcal{G}$ as the class of linear models where $g(\X) = \mathbf{w}^\top \X$, (2) perturbations of the form $\X_\xi = \X_0 \odot \xi$ where $\xi \, (\in \{0, 1\}^d) \sim \pi_{\X_0}$ with $\pi_{\X_0}$ being the exponential kernel (defined below), and (3) loss function as squared-error loss given by $\ell(f, g, \X_0, \xi) = (f(\X_\xi) - g(\xi))^2$ is equivalent to \LIME. 

As defined in \cite{ribeiro2016lime} (Section 3.4), the exponential kernel $\pi_{\X_0}(\xi) \propto exp \{ - \frac{D(\X_0, \X_\xi)}{\sigma^2} \}$ with distance function $D$ (such as cosine distance or L2 distance) and width $\sigma$.

\begin{proof}
For this instance of the LFA framework, by definition, the interpretable model $g$ is given by:

\begin{math}
    \begin{aligned}
        g^* &= \argmin_{g \in \mathcal{G}} \E_{\xi \sim \pi_{\X_0}} \ell(f, g, \X_0, \xi) \\
        &= \argmin_{g \in \mathcal{G}} \E_{\xi \sim p} [\ell(f, g, \X_0, \xi) \cdot \pi_{\X_0}(\xi)] \text{ where $p$ is the Bernouilli(0.5) distribution}\\
    \end{aligned}
\end{math}

Through importance sampling using a Bernouilli(0.5) proposal distribution (i.e., a Uniform(0,1) distribution over the space of binary inputs), the optimization setting of the LFA framework is that described for \LIME by \citet{ribeiro2016lime} (Equations 1 and 2).
\end{proof}

\subsubsection{KernelSHAP}\label{app:proof-kernelshap}

The instance of the LFA framework with (1) interpretable model class $\mathcal{G}$ as the class of linear models where $g(\X) = \mathbf{w}^\top \X$, (2) perturbations of the form $\X_\xi = \X_0 \odot \xi$ where $\xi \, (\in \{0, 1\}^d) \sim \pi$ with $\pi$ being the Shapley kernel (defined below), and (3) loss function as squared-error loss given by $\ell(f, g, \X_0, \xi) = (f(\X_\xi) - g(\xi))^2$ is equivalent to \KSHAP. 

As defined in \cite{lundberg2017shap} (Theorem 2), the Shapley kernel $\pi(\xi) \propto \frac{M-1}{{M \choose k} \cdot k \cdot (M-k)}$ where $M$ is the total number of elements in $\xi$ and $k$ is the number of ones in $\xi$.

\begin{proof}
For this instance of the LFA framework, by definition, the interpretable model $g$ is given by:

\begin{math}
    \begin{aligned}
        g^* &= \argmin_{g \in \mathcal{G}} \E_{\xi \sim \pi} \ell(f, g, \X_0, \xi) \\
        &= \argmin_{g \in \mathcal{G}} \E_{\xi \sim p} [\ell(f, g, \X_0, \xi) \cdot \pi(\xi)] \text{ where $p$ is the Bernouilli(0.5) distribution}\\
    \end{aligned}
\end{math}

Through importance sampling using a Bernouilli(0.5) proposal distribution (i.e., a Uniform(0,1) distribution over the space of binary inputs), the optimization setting of the LFA framework is that described for \KSHAP by \citet{lundberg2017shap} (Equation 2 and Theorem 2).
\end{proof}

\subsubsection{Occlusion}

The instance of the LFA framework with (1) interpretable model class $\mathcal{G}$ as the class of linear models where $g(\X) = \mathbf{w}^\top \X$, (2) perturbations of the form $\X_\xi = \X_0 \odot \xi$ where $\xi \, (\in \{0, 1\}^d)$ is a random one-hot vector, and (3) loss function as squared-error loss given by $\ell(f, g, \X_0, \xi) = (\Delta f - g(\xi))^2$ where $\Delta f =  f(\X_0) - f(\X_0 (1-\xi))$ converges to \Occlusion.

\begin{proof}
This instance of the LFA framework optimizes $g(\xi)$ to approximate $\Delta f$. For $\xi_i$ (a one-hot vector with element $i$ equal to 1), $g(\xi_i) = w_i$ and $\Delta f_i$ is the difference in the model prediction when feature $i$ takes the original value versus when feature $i$ is set to zero. $\Delta f$ is the definition of explanations generated by \Occlusion. Thus, in this instance of the LFA framework, the weights of $g$ recover the explanations of \Occlusion.
\end{proof}

\subsubsection{C-LIME}

The instance of the LFA framework with (1) interpretable model class $\mathcal{G}$ as the class of linear models where $g(\X) = \mathbf{w}^\top \X$, (2) perturbations of the form $\X_\xi = \X_0 + \xi$ where $\xi \, (\in \mathbb{R}^d) \sim \text{Normal}(0, \sigma^2)$, and (3) loss function as squared-error loss given by $\ell(f, g, \X_0, \xi) = (f(\X_\xi) - g(\xi))^2$ is equivalent to \CLIME.

\begin{proof}
This instance of the LFA framework is equivalent to \CLIME by definition of \CLIME.
\end{proof}

\subsubsection{SmoothGrad}\label{proof-smoothgrad}

In this section, we provide two derivations showing the connection between the LFA framework and \SmoothGrad. When using gradient-matching loss, the instance of the LFA framework is exactly equivalent to \SmoothGrad given the same $n$ perturbations. When using squared-error loss, the instance of the LFA framework is equivalent to \SmoothGrad asymptotically for a large number of perturbations.

\textbf{Gradient-matching loss function}

This instance of the LFA framework with 
(1) interpretable model class $\mathcal{G}$ as the class of linear models where $g(\X) = \mathbf{w}^\top \X$, 
(2) perturbations of the form $\X_\xi = \X_0 + \xi$ where $\xi \, (\in \mathbb{R}^d) \sim \text{Normal}(0, \sigma^2)$, and 
(3) loss function as gradient-matching loss given by $\ell_{gm}(f, g, \X_0, \xi) = \| \nabla_{\xi} f(\X_\xi) - \nabla_{\xi} g(\xi) \|_2^2$ 
is equivalent to \SmoothGrad. In other words, for the same $n$ perturbations, this instance of the LFA framework and \SmoothGrad yield the same explanation.

\begin{proof}
For this instance of the LFA framework, by definition, the interpretable model $g$ is given by $g^* = \argmin_{g \in \mathcal{G}} L$ where:

\begin{math}
\begin{aligned}
    L &= \mathbb{E}_\xi \ell(f, g, \X_0, \xi) \\
    &= \frac{1}{n} \sum_{n} \| \nabla_{\xi} f(\X_\xi) - \nabla_{\xi} g(\xi) \|_2^2 \\
    &= \frac{1}{n} \sum_{n} \| \nabla_{\X_0} f(\X_\xi) -  \mathbf{w} \|_2^2
\end{aligned}
\end{math}

To derive the solution for $\mathbf{w}$, take the partial derivative of $L$ w.r.t. $\mathbf{w}$, set the partial derivative to zero, and solve for $\mathbf{w}$. 

\begin{math}
\begin{aligned}
\nabla_{\mathbf{w}} L = 0 \Rightarrow 
(-2) \frac{1}{n} \sum_{n} [ \nabla_{\X_0} f(\X_\xi) -  \mathbf{w} ] = 0 \Rightarrow
\mathbf{w} = \frac{1}{n} \sum_{n} \nabla_{\X_0} f(\X_\xi)
\end{aligned}
\end{math}

Therefore, for the same $n$ perturbations, the weights $\mathbf{w}$ of the interpretable model $g$ are equivalent to the \SmoothGrad explanations.
\end{proof}

\textbf{Squared-error loss function}

Consider the instance of the LFA framework corresponding to \SmoothGrad described above, except with loss function as squared-error loss given by $\ell(f, g, \X_0, \xi) = (f(\X_\xi) - g(\xi))^2$. This instance of the LFA framework converges to \SmoothGrad in expectation. Note that this instance of the LFA framework is \CLIME and its convergence to \SmoothGrad in expectation is consistent with the results of \cite{agarwal2021clime} which previously proved the same convergence (using a different approach).

\begin{proof}
For this instance of the LFA framework, by definition, the interpretable model $g$ is given by $g^* = \argmin_{g \in \mathcal{G}} L$ where:

\begin{math}
\begin{aligned}
    L &= \mathbb{E}_\xi \ell(f, g, \X_0, \xi) \\
    &= \mathbb{E}_\xi [ (f(\X_\xi) - g(\xi))^2 ] \\
    &= \mathbb{E}_\xi [ (f(\X_\xi) - \mathbf{w}^\top \xi)^2 ] \\
\end{aligned}
\end{math}

To derive the solution for $\mathbf{w}$, take the partial derivative of $L$ w.r.t. $\mathbf{w}$, set the partial derivative to zero, and solve for $\mathbf{w}$. 

\begin{math}
\begin{aligned}
    \nabla_{\mathbf{w}} L &= 0 \\
    -2 \mathbb{E}_\xi[(f(\X_\xi) - \mathbf{w}^\top \xi) \xi^\top] &= 0 \\
    \mathbb{E}_\xi[f(\X_\xi) \xi^\top - \mathbf{w}^\top \xi \xi^\top] &= 0 \\
    \mathbb{E}_\xi[f(\X_\xi) \xi^\top] - \mathbf{w}^\top \mathbb{E}[\xi \xi^\top] &= 0 \\
    \sigma^2 \mathbb{E}_\xi[\nabla_{\X_\xi} f(\X_\xi)^\top] - \sigma^2 \mathbf{w}^\top &= 0 \textnormal{ by Stein's Lemma}\\
    \sigma^2 \mathbb{E}_\xi[\nabla_{\X_0} f(\X_\xi)^\top] - \sigma^2 \mathbf{w}^\top &= 0\\
    \mathbf{w} &= \mathbb{E}_\xi[\nabla_{\X_0} f(\X_\xi)] \\
\end{aligned}
\end{math}

Therefore, the weights $\mathbf{w}$ of the interpretable model $g$ converge to \SmoothGrad explanations in expectation.
\end{proof}


\subsubsection{Vanilla gradients}

Consider the instance of the LFA framework corresponding to \SmoothGrad described above (with loss function as either squared-error loss or gradient-matching loss). As $\sigma \rightarrow 0$, this instance of the LFA framework converges to \Grads.

\begin{proof}
Starting with the solution for $\mathbf{w}$ derived for \SmoothGrad, take the limit of $\mathbf{w}$ as $\sigma \rightarrow 0^+$.

\begin{math}
\begin{aligned}
    \lim_{\sigma \to 0^+} \mathbf{w} &= \lim_{\sigma \to 0^+} \mathbb{E}_\xi[\nabla_{\X_0} f(\X_\xi)] \\
    &= \lim_{\sigma \to 0^+}  \int_{-\infty}^{\infty} \nabla_{\X_0} f(\X_\xi) \, p(\xi; 0, \sigma)\, d\xi \\
    &= \lim_{\sigma \to 0^+} \int_{-\infty}^{\infty} \nabla_{\X_0} f(\X_0 + \xi) \, \eta_\xi(\xi)\, d\xi \text{ where $\eta_\xi(\xi)=p(\xi; 0, \sigma)$} \\
    &= \nabla_{\X_0} f(\X_0) \textnormal{ by property of the Dirac delta distribution}\\
\end{aligned}
\end{math}

To derive the third line from the second line, we view the Normal density function $p(\xi; 0, \sigma)$ as a nascent delta function $\eta_\xi(\xi)$ (which is defined such that $\lim_{\sigma \to 0^+} \int_{-\infty}^{\infty} p(\xi; 0, \sigma) = \delta(\xi)$, where $\delta$ is the Dirac delta distribution) and by assuming that $\nabla_{\X_0} f(\X_\xi)$ has a compact support.


Therefore, the weights $\mathbf{w}$ of the interpretable model $g$ converge to \Grads explanations.
\end{proof}

\subsubsection{Integrated Gradients}

This instance of the LFA framework with 
(1) interpretable model class $\mathcal{G}$ as the class of linear models where $g(\X) = \mathbf{w}^\top \X$, 
(2) perturbations of the form $\X_\xi = \X_0 \odot \xi$ where $\xi \, (\in \mathbb{R}^d) \sim \text{Uniform}(0, 1)$, and 
(3) loss function as gradient-matching loss given by $\ell_{gm}(f, g, \X_0, \xi) = \| \nabla_{\xi} f(\X_\xi) - \nabla_{\xi} g(\xi) \|_2^2$ 
is equivalent to \IntGrad. In other words, for the same $n$ perturbations, this instance of the LFA framework and \IntGrad yield the same explanation.

\begin{proof}
For this instance of the LFA framework, by definition, the interpretable model $g$ is given by $g^* = \argmin_{g \in \mathcal{G}} L$ where:

\begin{math}
\begin{aligned}
    L &= \mathbb{E}_\xi \ell(f, g, \X_0, \xi) \\
    &= \mathbb{E}_\xi \| \nabla_{\xi} f(\X_\xi) - \nabla_{\xi} g(\xi) \|_2^2  \\
    &= \mathbb{E}_\xi \| \nabla_{\X_0} f(\X_\xi) \odot \X_0 - \mathbf{w} \|_2^2 \\
\end{aligned}
\end{math}

Note that, by the chain rule, $\nabla_{\xi} f(\X_\xi) = \nabla_{\xi} f(\X_0 \odot \xi) =  \nabla_{\X_\xi} f(\X_\xi)  \odot \nabla_{\xi} \X_\xi = \nabla_{\X_0} f(\X_\xi)  \odot \X_0$.

To derive the solution for $\mathbf{w}$, take the partial derivative of $L$ w.r.t. $\mathbf{w}$, set the partial derivative to zero, and solve for $\mathbf{w}$. 

\begin{math}
\begin{aligned}
    \nabla_{\mathbf{w}} L &= 0 \\
    -2 \mathbb{E}_\xi[\nabla_{\X_0} f(\X_\xi) \odot \X_0 - \mathbf{w}] &= 0 \\
    \mathbf{w} &= \X_0 \odot \mathbb{E}_\xi[\nabla_{\X_0} f(\X_\xi)]\\
\end{aligned}
\end{math}

Therefore, the weights $\mathbf{w}$ of the interpretable model $g$ converge to \IntGrad explanations in expectation.
\end{proof}

\subsubsection{Gradient $\times$ Input}

Consider the instance of the LFA framework corresponding to \IntGrad described above, except with $\xi \, (\in \mathbb{R}^d) \sim \text{Uniform}(a, 1)$. As $a \rightarrow 1$, this instance of the LFA framework converges to \GradtimesInput.

\begin{proof}
As $a \rightarrow 1$, $\xi \rightarrow \vec{1}$, and $\mathbf{w} \rightarrow \X_0  \odot \nabla_{\X_0} f(\X_0)$. Therefore, the weights $\mathbf{w}$ of the interpretable model $g$ converge to \GradtimesInput explanations.
\end{proof}

\subsection{Which Explanations Are Not Function Approximations?}\label{app:notLFA}

In this section, we briefly discuss explanation methods that cannot be viewed as instances of the LFA framework. In the cases below, the lack of connection to the LFA framework is mainly due to a property of the explanation method.

\textbf{Model-independent methods.} Some explanation methods are known to produce attributions that are independent of the model they intend to explain. These methods cannot be cast in the LFA framework in a meaningful way due to the model recovery conditions we impose. Such model-independent methods include guided backpropagation \cite{springenberg2015striving} and DeconvNet \cite{noh2015learning}, following theory by \citet{nie2018theoretical}, as well as logit-gradient-based methods \cite{srinivas2021rethinking} such as Grad-CAM \cite{selvaraju2017grad}, Grad-CAM++ \cite{chattopadhay2018grad}, and FullGrad \cite{srinivas2019full}.


\textbf{Modified-backpropagation methods.} Some explanation methods such as DeepLIFT \cite{shrikumar2017learning}, guided backpropagation \cite{springenberg2015striving}, DeconvNet \cite{noh2015learning}, and layer-wise relevance propagation \cite{bach2015pixel} work by modifying the backpropagation equations and propagating attributions using finite-difference-like 
methods. Such methods violate an important property called ``implementation invariance'', first identified by \citet{sundararajan2017integratedgrad}, which states that two functionally 
identical models can have different attributions due to the lack of a chain rule for modified backpropagation methods. This property ensures that such methods cannot be function approximators, as the attribution changes based on the function implementation.


\textbf{Unsigned-gradient methods.} Some gradient-based methods return unsigned attribution values instead of the full signed values. Such methods can be written in the LFA framework using the following loss function $\ell(f, g, \X_0, \xi) = \| |\grad f(\X_0 \oplus \xi)| - \mathbf{w}_g \|^2$ where $\mathbf{w}_g$ consists of the weights of the interpretable model $g$. Using this loss function with different choices for neighborhoods gives unsigned versions of 
different gradient methods. However, this loss function is not a valid loss function, i.e., $\ell = 0 \centernot\implies f = g$. Using this loss function, $\mathbf{w}_g$ is always positive and thus cannot recover an underlying model's negative weights.

\subsection{Proof: No Free Lunch Theorem (Section~\ref{sec:whichmethod})}\label{app:nfl-theorem}

\begin{thm*}
Consider explaining a black-box model $f$ around point $\X_0$ using an interpretable model $g$ from model class $\mathcal{G}$ and a valid loss function $\ell$ where the distance between $f$ and $\mathcal{G}$ is given by $d(f, \mathcal{G}) = \min_{g \in \mathcal{G}} \max_{\X \in \mathcal{X}} \ell(f,g,0,\X)$.

Then, for any explanation $g^*$ over a neighbourhood distribution $\xi_1~\sim~ \mathcal{Z}_1$ such that $\max_{\xi_1} \ell(f,g^*,\X_0, \xi_1) \leq \epsilon$, there always exists another neighbourhood $\xi_2 \sim \mathcal{Z}_2$ such that $\max_{\xi_2} \ell(f,g^*,\X_0, \xi_2) \geq d(f, \mathcal{G})$.
\end{thm*}

\begin{proof}
Given an explanation $g^*$, we can find an "adversarial" input $\X_{adv}$ such that $\X_{adv} = \arg\max_{\X \in \mathcal{X}} \ell(f,g^*,0,\X)$ has a large error $\ell$. Construct perturbation $\X_2 = \X_0 + \xi_2$ such that $p(\xi_2) = \text{Uniform}(0, \X_{adv} - \X_0)$, which implies $p(\X_2) = \text{Uniform}(\X_0, \X_{adv})$. In this proof, $\text{Uniform}(a, b)$ denotes uniformly sampling along the straight line connecting $a$ and $b$.

By definition $\max_{\xi_2} \ell(f, g^*, \X_0, \xi_2) = \ell(f, g^*, \X_0, \X_{adv} - \X_0) = \max_{\X \in \mathcal{X}} \ell(f, g^*, 0, \X) \geq \min_{g \in \mathcal{G}} \max_{\X \in \mathcal{X}} \ell(f, g, 0, \X) = d(f, \mathcal{G})$
\end{proof}

A salient feature of this proof is that it makes no assumptions about the form of model, input or output domains. This implies that the result applies equally to discrete and continuous domains, regression and classification tasks, and for any model type.

\subsection{Summary of Properties of Existing Explanation Methods}
\label{app:summary-model-recovery}

\begin{table}[H]
    \centering
    \footnotesize
    \begin{tabular}{c|c|c|c}
        \textbf{Method} & \textbf{Characteristics of $\xi$} & \textbf{$g$ recovers $f$?} & \textbf{Scale of $g$'s weights when $\mathcal{X} \in \mathbb{R}^d$} \\
        \midrule
         C-LIME & Continuous, Additive & When $\mathcal{X} \in \mathbb{R}^d$ & Gradient \\
         SmoothGrad & Continuous, Additive & When $\mathcal{X} \in \mathbb{R}^d$ & Gradient \\
         Vanilla Gradients & Continuous, Additive & When $\mathcal{X} \in \mathbb{R}^d$ & Gradient \\
         \midrule
         Integrated Gradients & Continuous, Multiplicative & No & Gradient $\times$ Input \\
         Gradients $\times$ Input & Continuous, Multiplicative & No & Gradient $\times$ Input \\
         \midrule 
         LIME & Binary, Multiplicative & When $\mathcal{X} \in \{0, 1\}^d$ & Gradient $\times$ Input \\
         KernelSHAP & Binary, Multiplicative & When $\mathcal{X} \in \{0, 1\}^d$ & Gradient $\times$ Input \\
         Occlusion & Binary, Multiplicative & When $\mathcal{X} \in \{0, 1\}^d$ & Gradient $\times$ Input \\
    \end{tabular}
    \caption{Summary of properties of existing explanation methods in relation to the LFA framework. In this table, we consider the scale of $g$'s weights when $\mathcal{X} \in \mathbb{R}^d$.}  
    \label{table:summary}
\end{table}

\subsection{Setup of Experiments} \label{app:exp-info}

\textbf{Datasets.} The first dataset is the life expectancy dataset from the Global Health Observatory data repository of the World Health Organization (WHO) \cite{dataset2018who}. The WHO dataset consists of demographic, economic, and health factors of 193 countries from 2000 to 2015, including a country's population, gross domestic product, health expenditure, human development index, infant mortality rate, hepatitis B immunization rate, and life expectancy. The other dataset is the home equity line of credit (HELOC) dataset from the Explainable Machine Learning Challenge organized by FICO \cite{dataset2019heloc}. The HELOC dataset contains information on HELOC applications made by homeowners, including an applicant's installment balance, number of trades, longest delinquency period, and risk category (whether an applicant made payments without being 90 days overdue). To our knowledge, these datasets do not contain personally identifiable information or offensive content.

For the WHO dataset, missing values were imputed using kNN imputation with $k=5$. For the HELOC dataset, missing values were dropped. For both datasets, continuous features were mean-centered and then normalized to [0, 1] range.

\textbf{Models.} 
For the WHO dataset, we train four models: a linear regression model (train MSE: $9.39 \times 10^{-5}$; test MSE: $9.80 \times 10^{-5}$) and three feed-forward neural networks. The neural networks have 8-node hidden layers with tanh activation and a linear output layer. The first neural network has 3 hidden layers (train MSE: $7.83 \times 10^{-5}$; test MSE: $8.23 \times 10^{-5}$), the second has 5 hidden layers (train MSE: $7.76 \times 10^{-5}$; test MSE: $8.11 \times 10^{-5}$), and the third has 8 hidden layers (train MSE: $7.78 \times 10^{-5}$; test MSE: $8.20 \times 10^{-5}$). The neural networks are referred to as NN1, NN2, and NN3, respectively.

For the HELOC dataset, we train four models: a logistic regression model (train accuracy: 0.73; test accuracy: 0.74) and three feed-forward neural networks. The neural networks have 8-node hidden layers with relu activation and an output layer with sigmoid activation. The first neural network has 3 hidden layers (train accuracy: 0.75; test accuracy: 0.75), the second has 5 hidden layers (train accuracy: 0.75; test accuracy: 0.75), and the third has 8 hidden layers (train accuracy: 0.75; test accuracy: 0.75). The neural networks are referred to as NNA, NNB, and NNC, respectively.

Models were trained based on an 80/20 train/test split using stochastic gradient descent. Hyperparameters were selected to reach decent model performance. The emphasis is on generating explanations for individual model predictions, not on high model performance. Thus, we do not focus on tuning model hyperparameters. Linear and logistic regression models trained for 100 epochs while neural network models trained for 300 epochs. All models used a batch size of 64 and a cosine annealing scheduler for the learning rate. Hyperparameters for all models are included in the code accompanying this paper.

\textbf{Explanation Methods.}
Each explanation method is implemented using (1) the existing method and (2) the LFA framework. For (1), we used Meta's Captum library \cite{kokhlikyan2020captum}. When using Captum, methods with number of perturbations as a parameter (i.e., LIME, KernelSHAP, SmoothGrad, and Integrated Gradients) used 1000 perturbations, a number of perturbations at which explanations for the method converged. For (2), we implemented the LFA framework, instantiating each method based on Table~\ref{table:meta-algo-instances}. For each method, the number of perturbations is set to 1000 for the same reason above. The interpretable model $g$ is optimized using stochastic gradient descent. The perturbations are split into a train and test set (80/20 split) and $g^*$ is optimized based on test set performance.

Analyses were performed on GPUs. The total amount of compute is approximately 54 GPU-hours.

\subsection{Full Results for Experiments}
\label{app:exp-results}

\subsubsection{Experiment 1: Existing Methods Are Instances of the LFA Framework}
\label{app:exp1-lfa-all}

\newpage
\begin{figure}[H]
    \centering
    \includegraphics[width=0.8\textwidth]{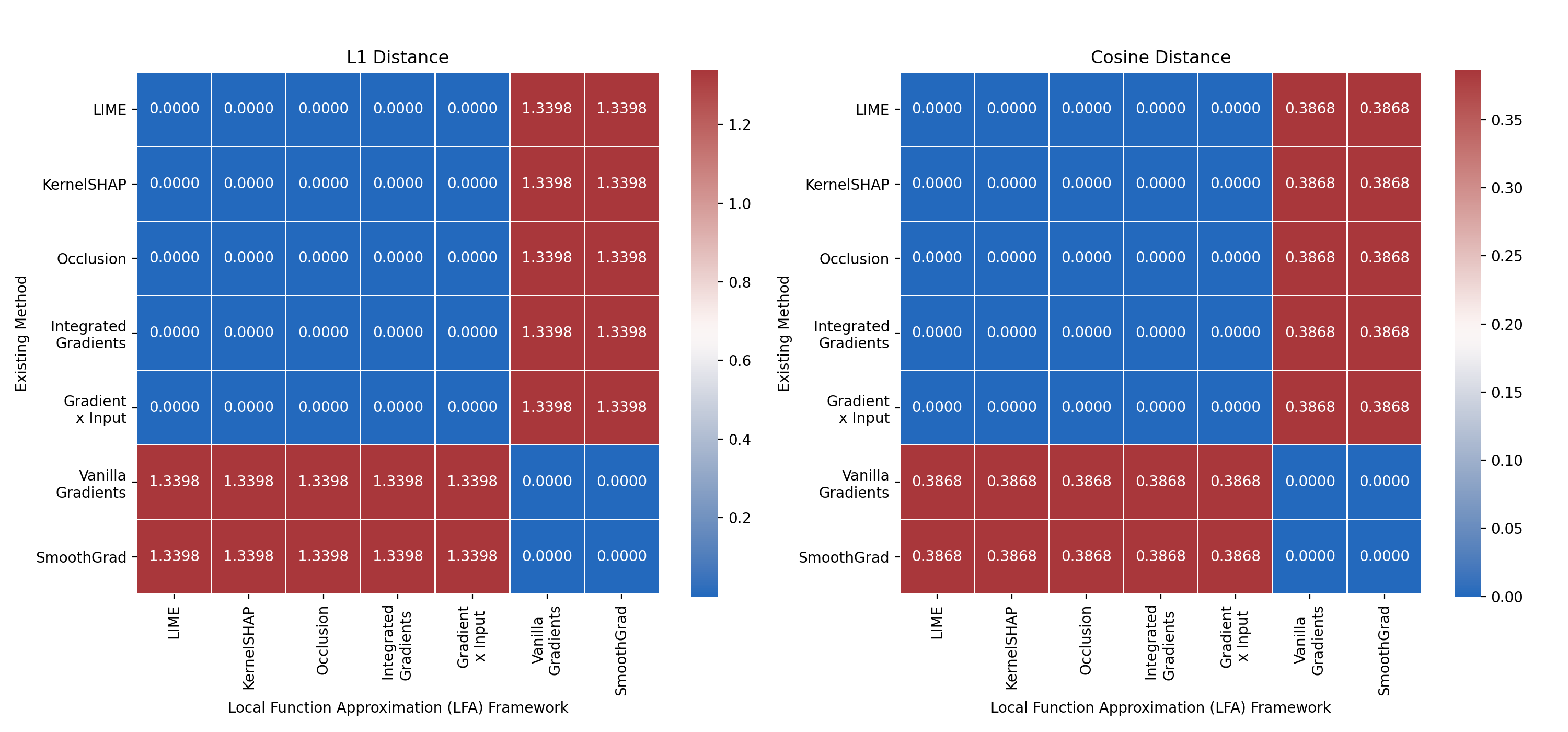}\vfill
    \includegraphics[width=0.8\textwidth]{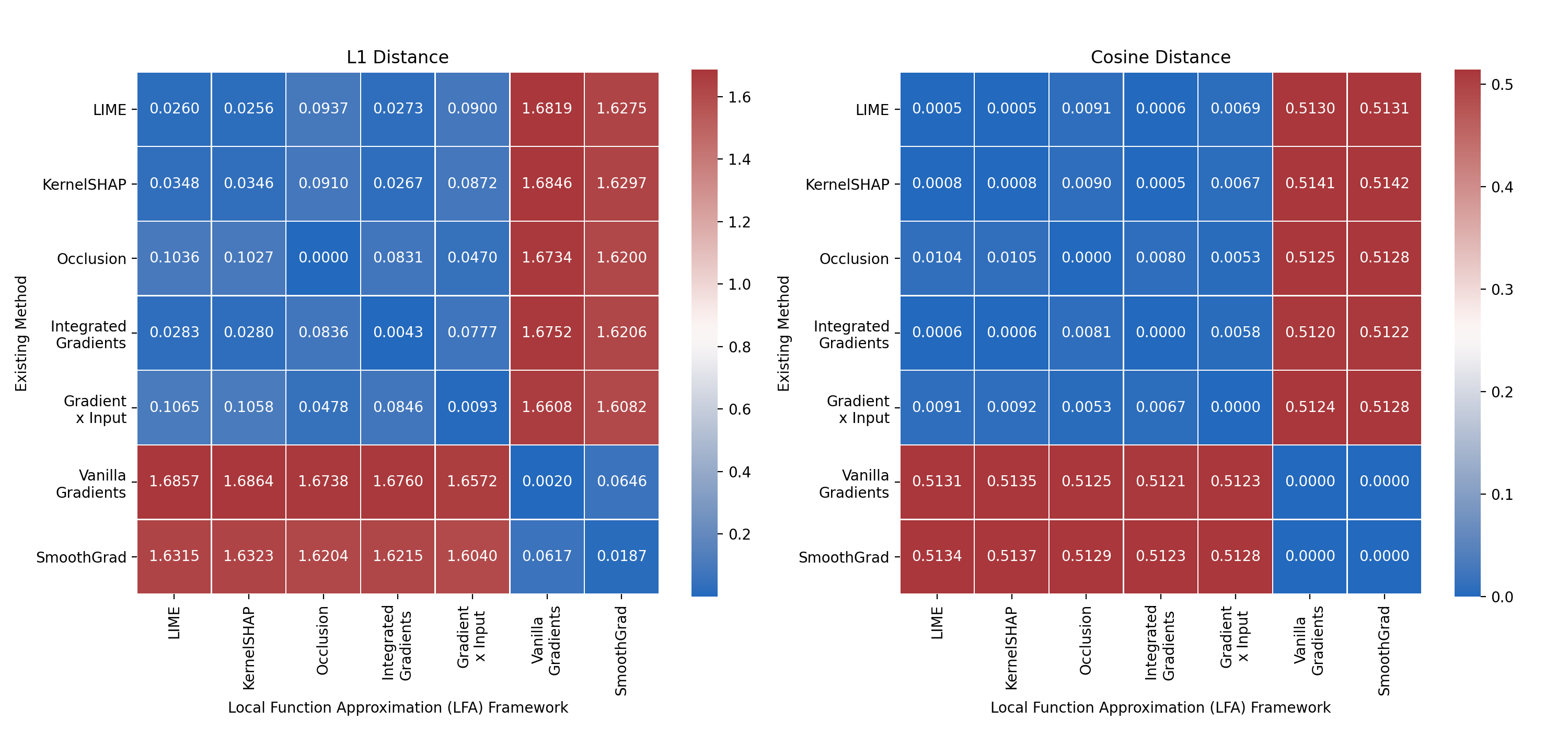}\vfill
    \includegraphics[width=0.8\textwidth]{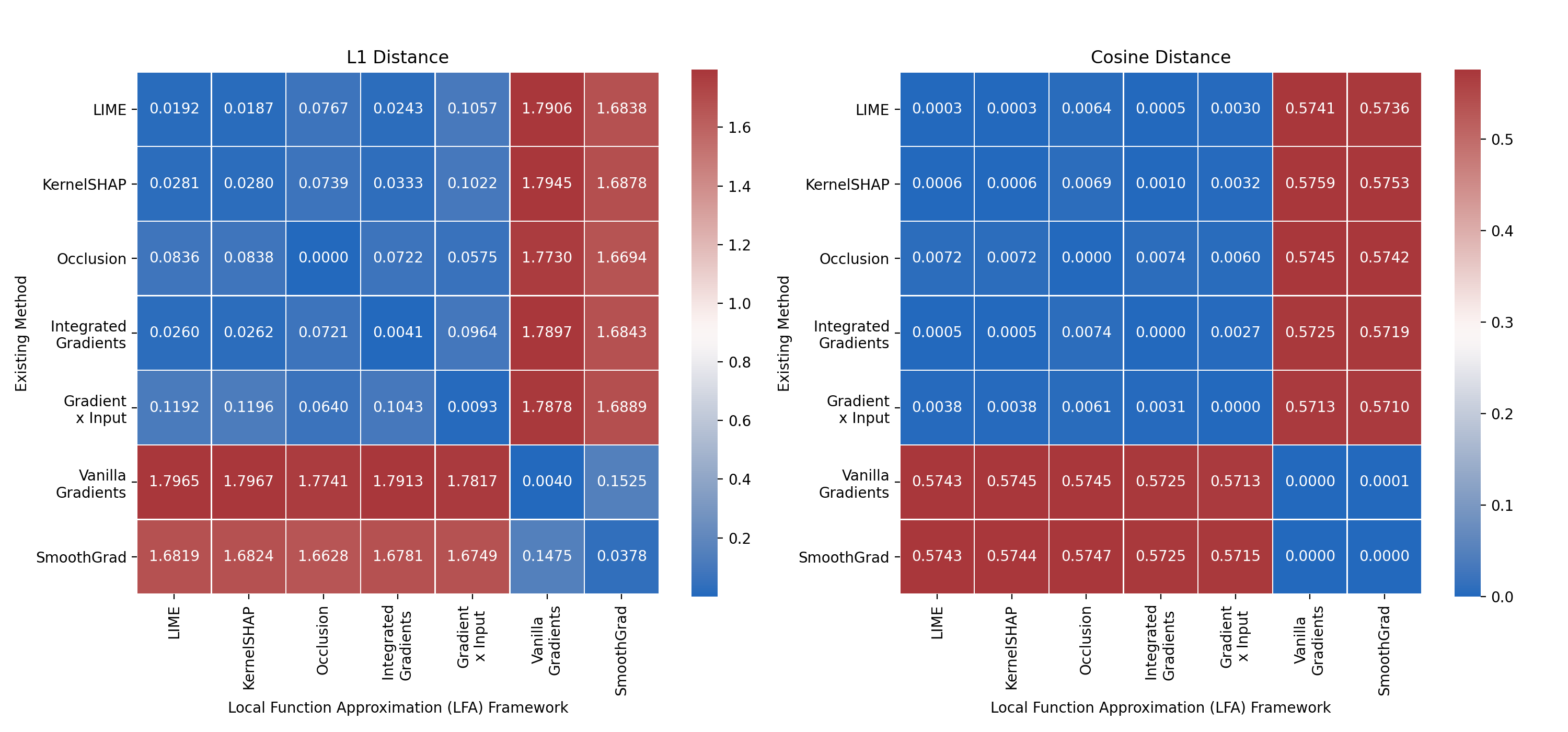}\vfill
    \includegraphics[width=0.8\textwidth]{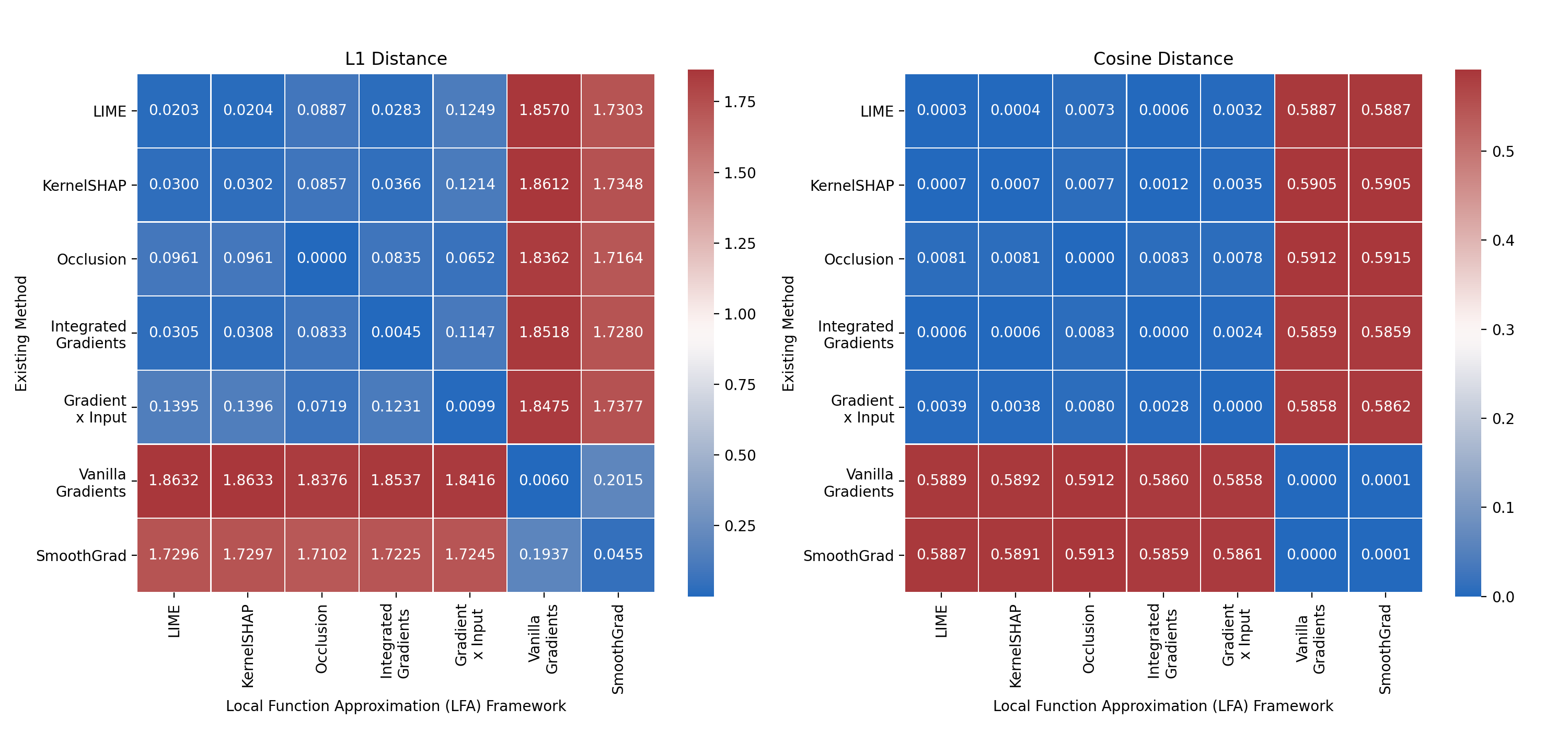}
    \caption{Correspondence of existing methods to instances of the LFA framework. Experiments performed on the WHO dataset for linear regression (Row 1), NN1 (Row 2), NN2 (Row 3), and NN3 (Row 4). The similarity of pairs of explanations are measured based on L1 distance (left column) and cosine distance (right column).}
\end{figure}

\begin{figure}[H]
    \centering
    \includegraphics[width=0.6\textwidth]{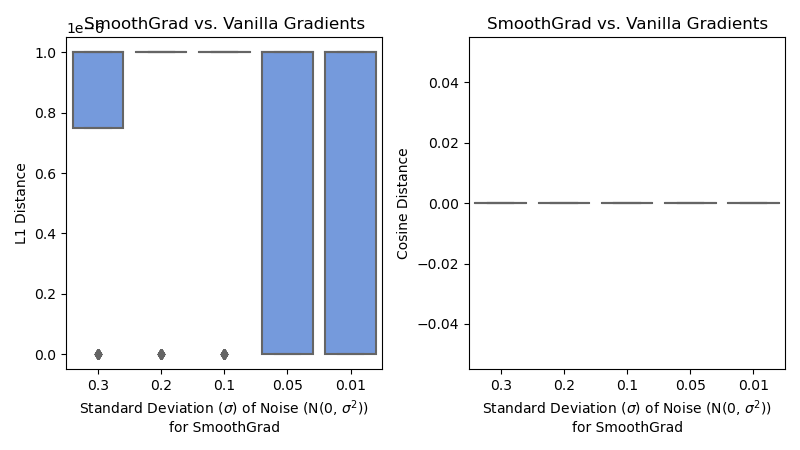}\vfill
    \includegraphics[width=0.6\textwidth]{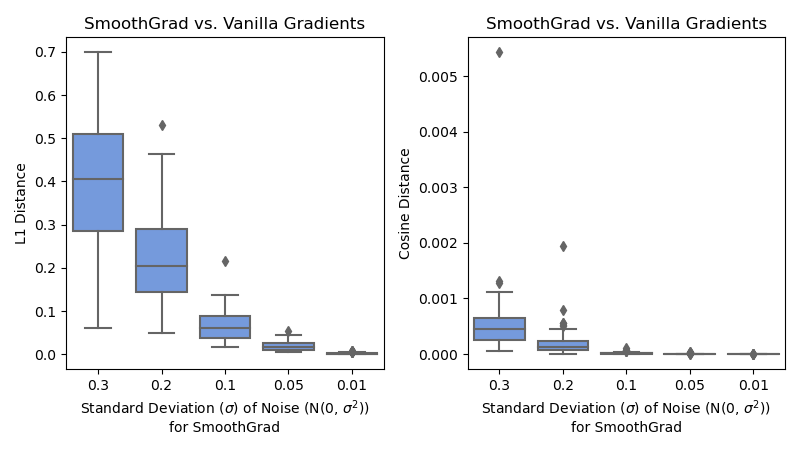}\vfill
    \includegraphics[width=0.6\textwidth]{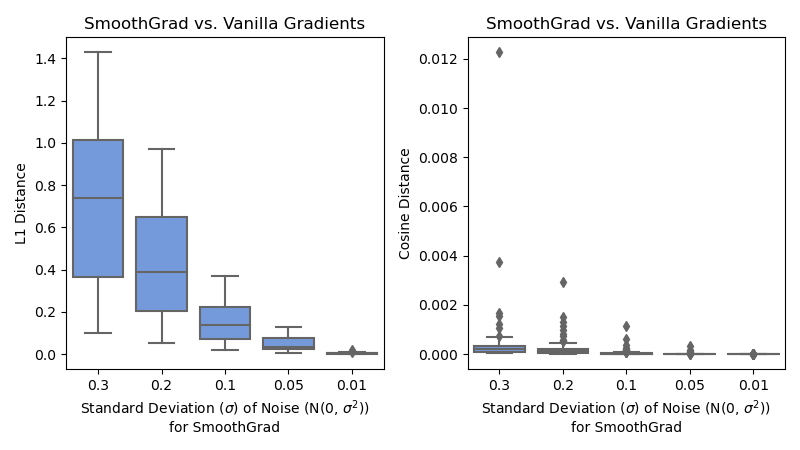}\vfill
    \includegraphics[width=0.6\textwidth]{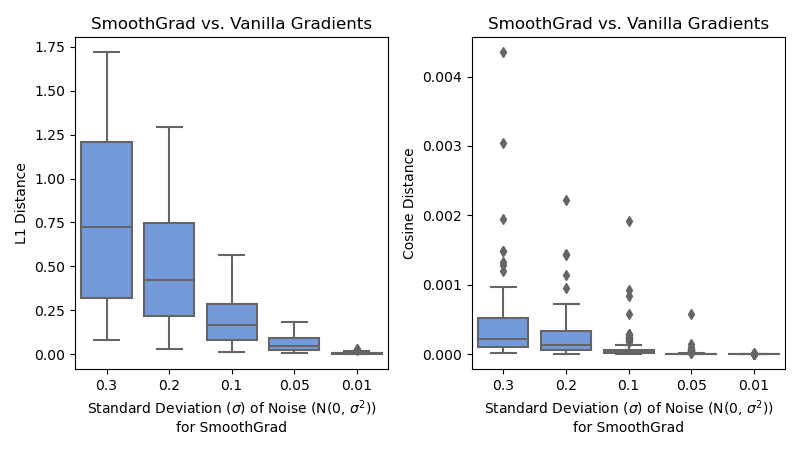}
    \caption{Using the LFA framework, explanations generated by SmoothGrad converge to those generated by Vanilla Gradients. Experiments performed on the WHO dataset for linear regression (Row 1), NN1 (Row 2), NN2 (Row 3), and NN3 (Row 4). The similarity of pairs of explanations are measured based on L1 distance (left column) and cosine distance (right column).}
\end{figure}

\begin{figure}[H]
    \centering
    \includegraphics[width=0.6\textwidth]{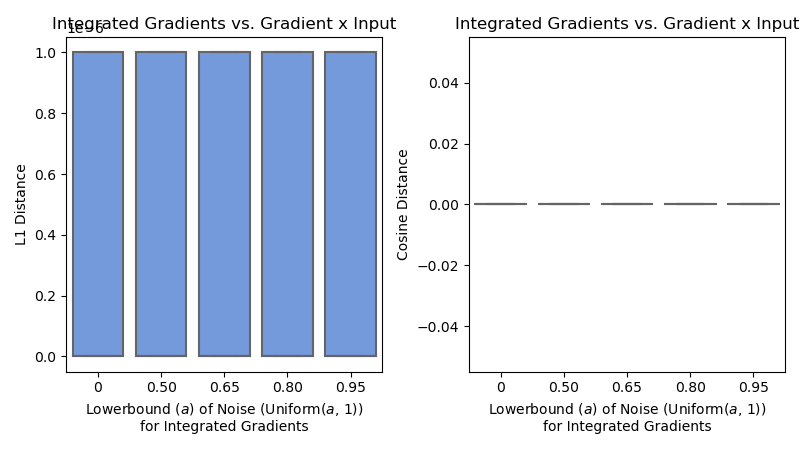}\vfill
    \includegraphics[width=0.6\textwidth]{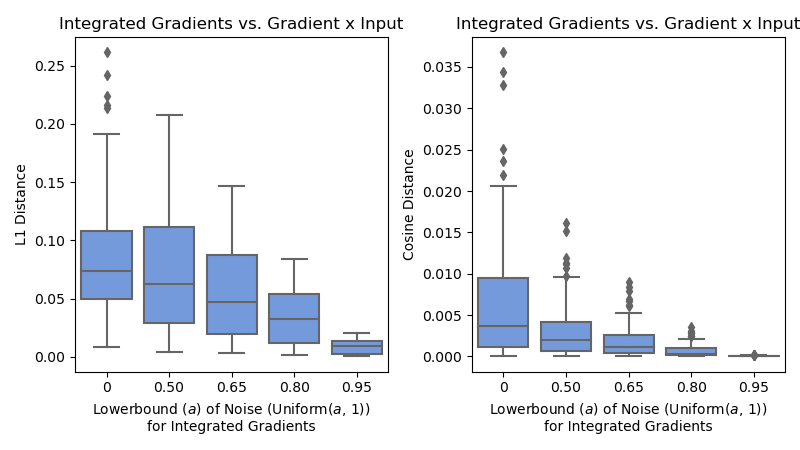}\vfill
    \includegraphics[width=0.6\textwidth]{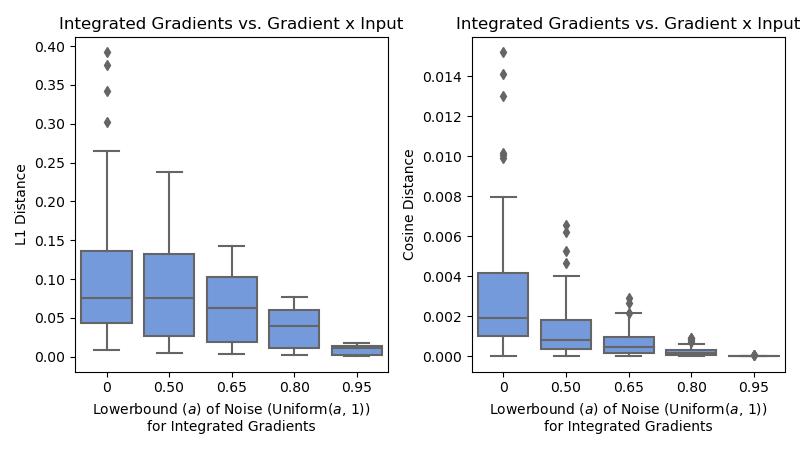}\vfill
    \includegraphics[width=0.6\textwidth]{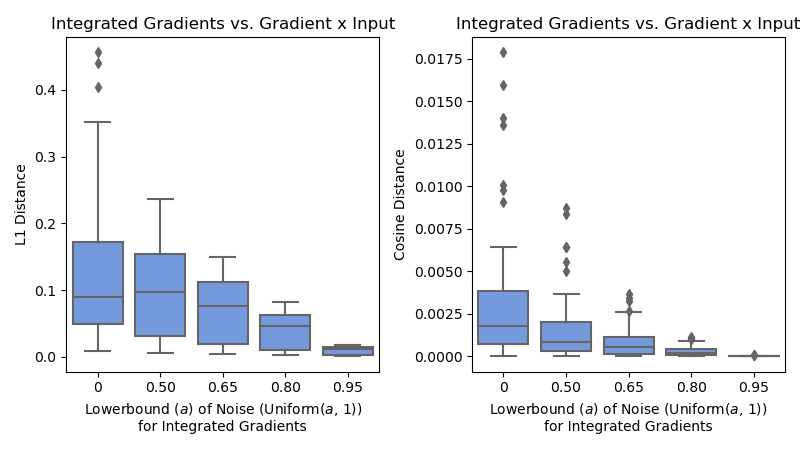}
    \caption{Using the LFA framework, explanations generated by Integrated Gradients converge to those generated by Gradient $\times$ Input. Experiments performed on the WHO dataset for linear regression (Row 1), NN1 (Row 2), NN2 (Row 3), and NN3 (Row 4). The similarity of pairs of explanations are measured based on L1 distance (left column) and cosine distance (right column).}
\end{figure}

\begin{figure}[H]
    \centering
    \includegraphics[width=0.8\textwidth]{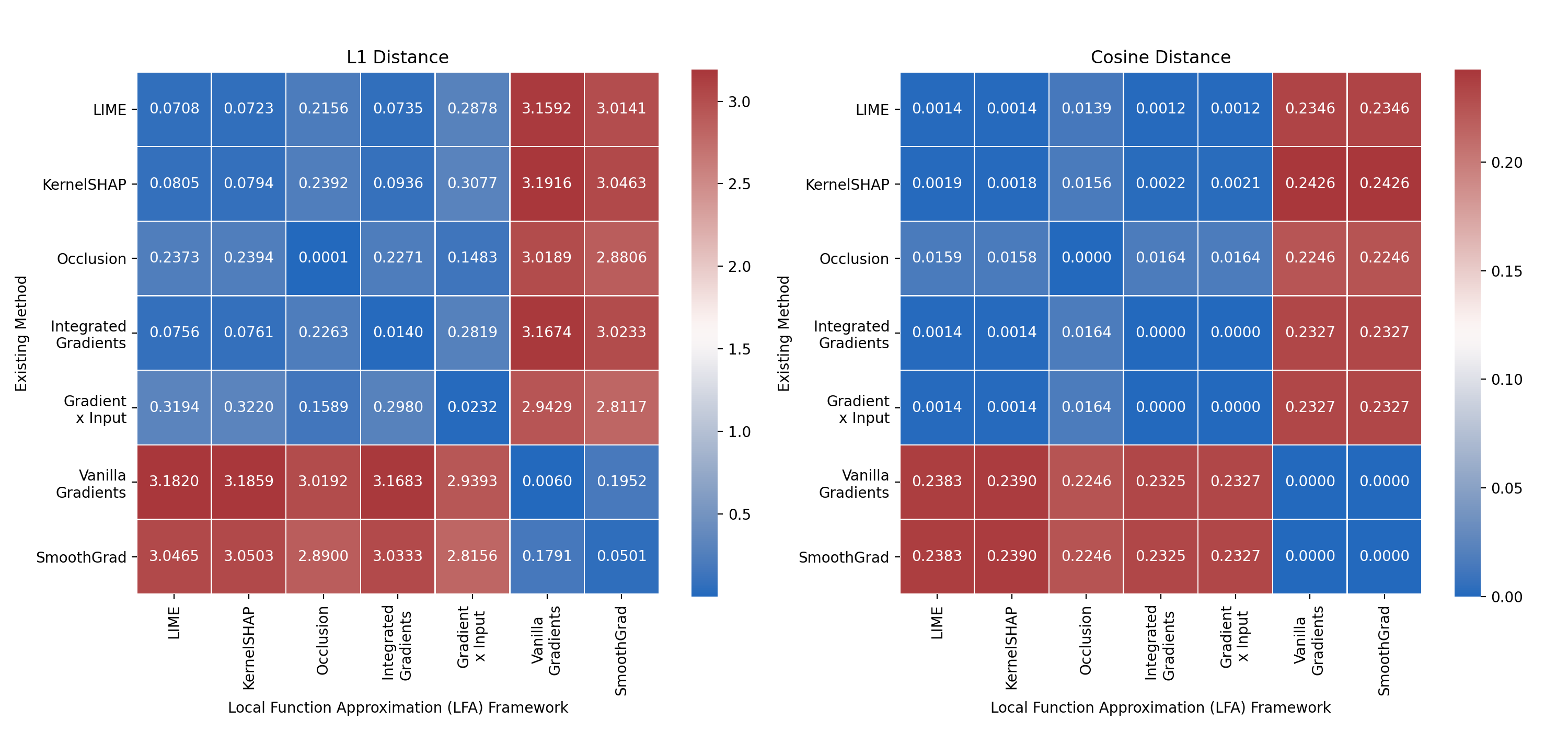}\vfill
    \includegraphics[width=0.8\textwidth]{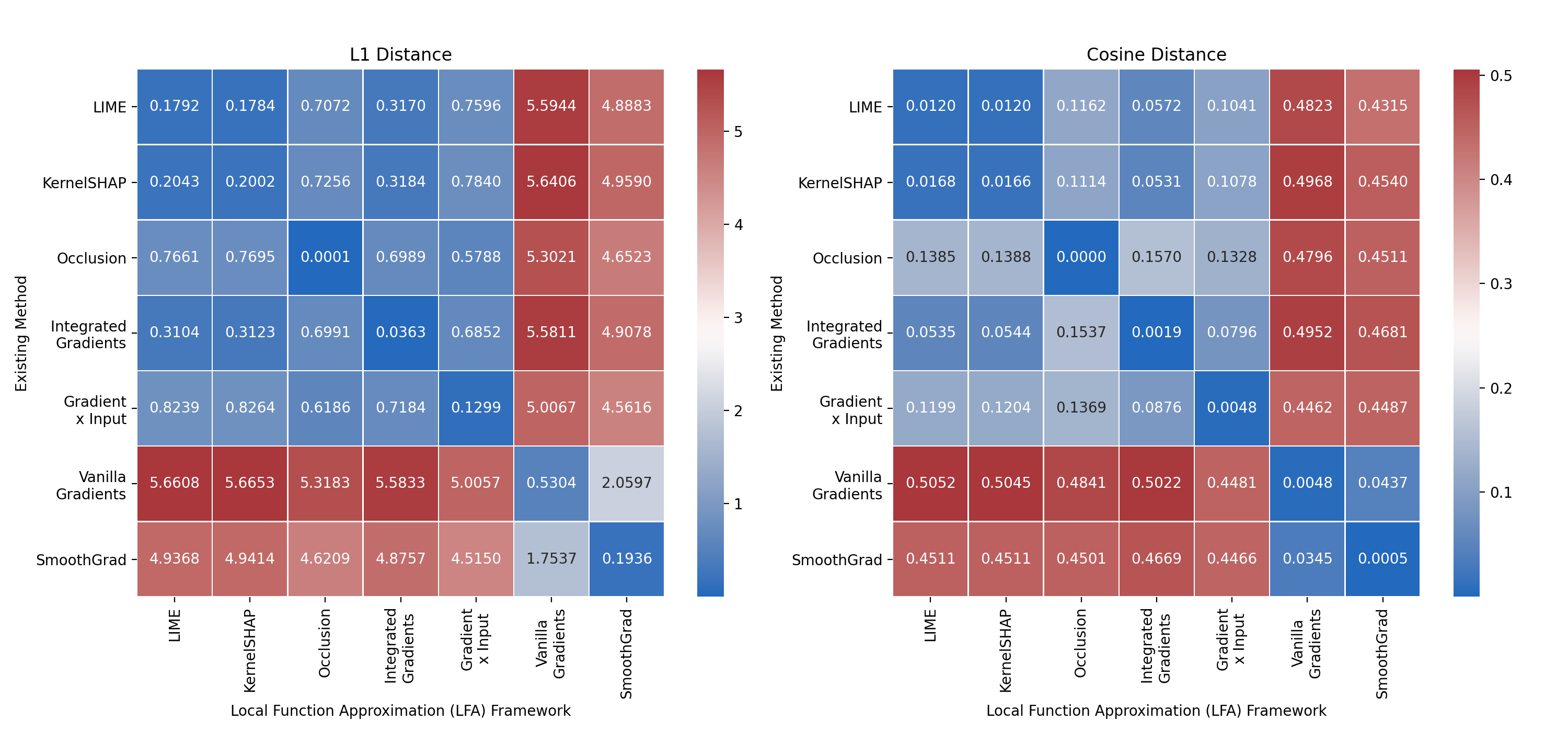}\vfill
    \includegraphics[width=0.8\textwidth]{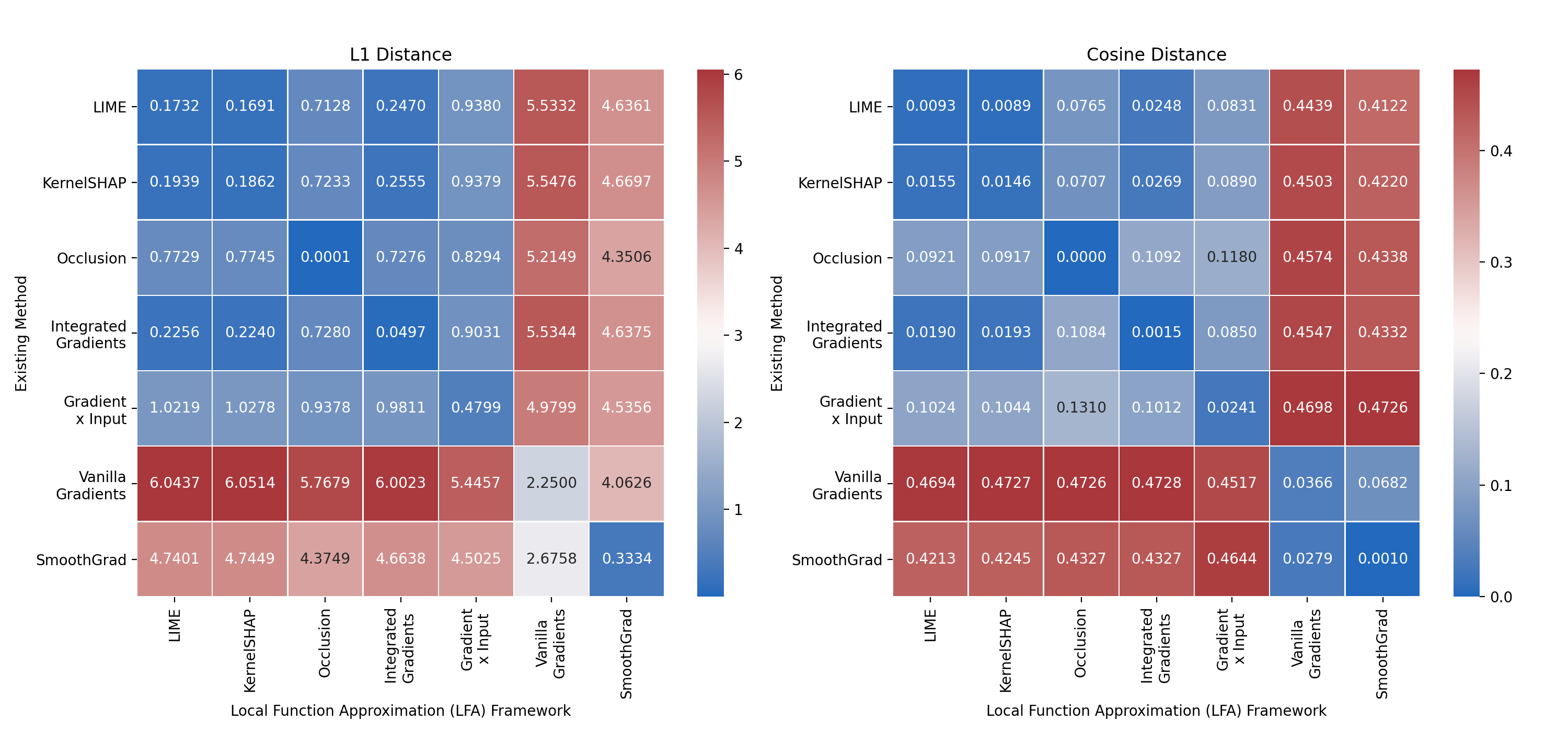}\vfill
    \includegraphics[width=0.8\textwidth]{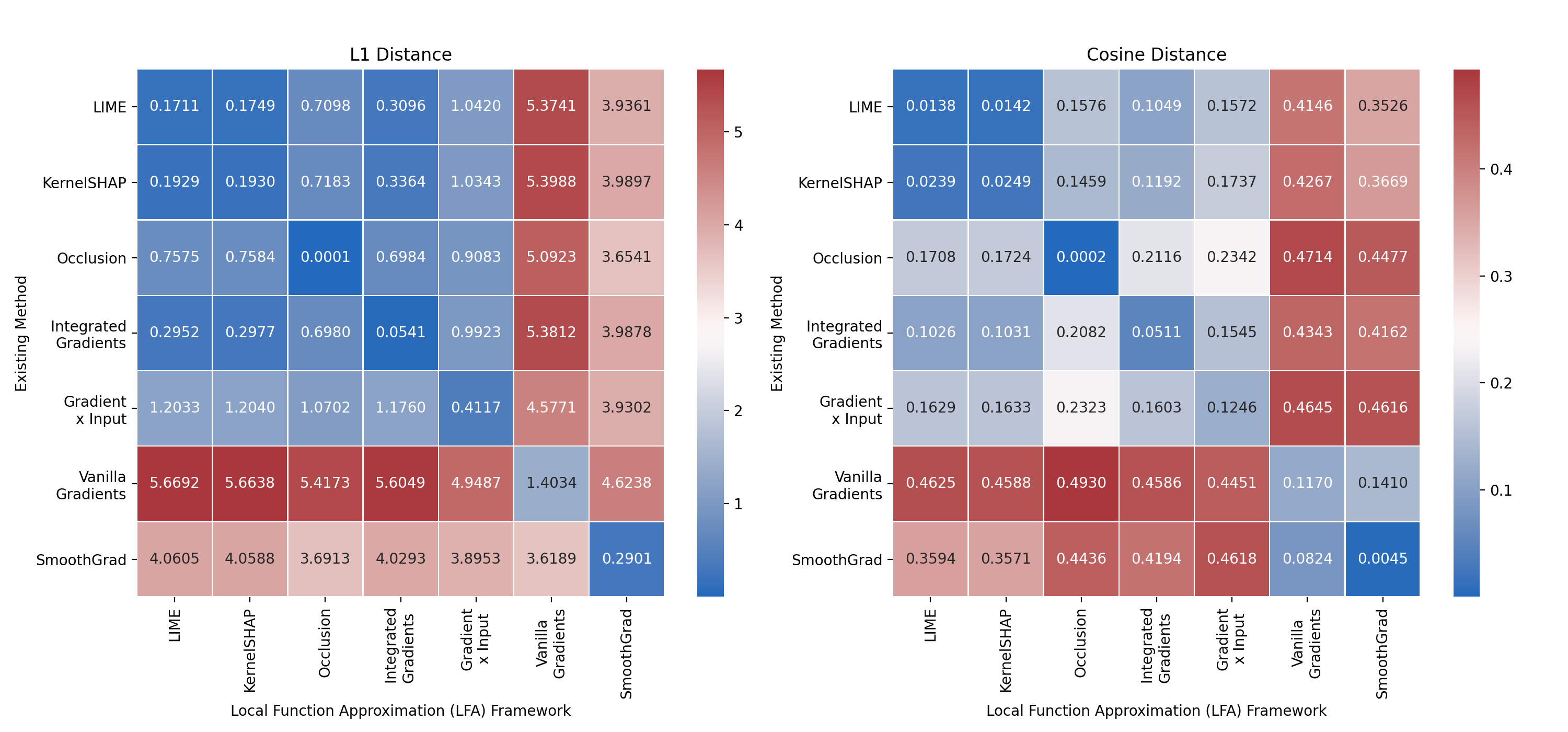}
    \caption{Correspondence of existing methods to instances of the LFA framework. Experiments performed on the HELOC dataset for logistic regression (Row 1), NNA (Row 2), NNB (Row 3), and NNC (Row 4). The similarity of pairs of explanations are measured based on L1 distance (left column) and cosine distance (right column).}
\end{figure}

\begin{figure}[H]
    \centering
    \includegraphics[width=0.6\textwidth]{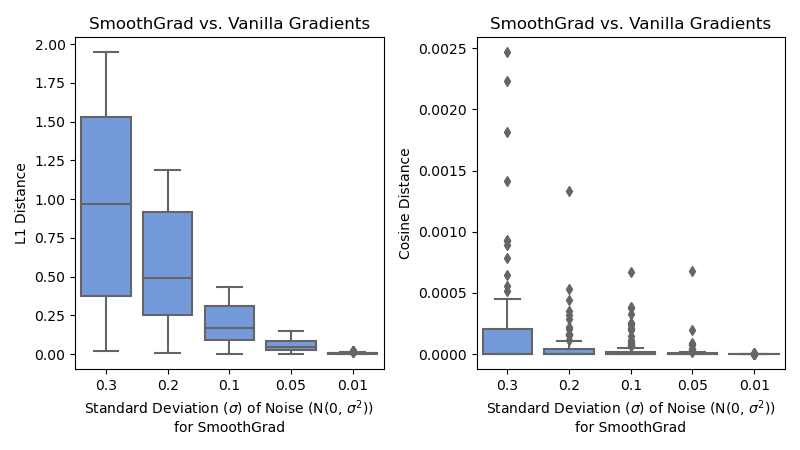}\vfill
    \includegraphics[width=0.6\textwidth]{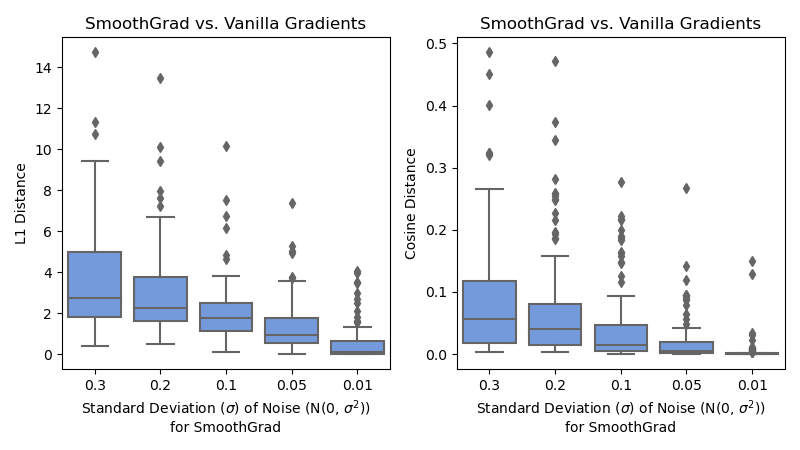}\vfill
    \includegraphics[width=0.6\textwidth]{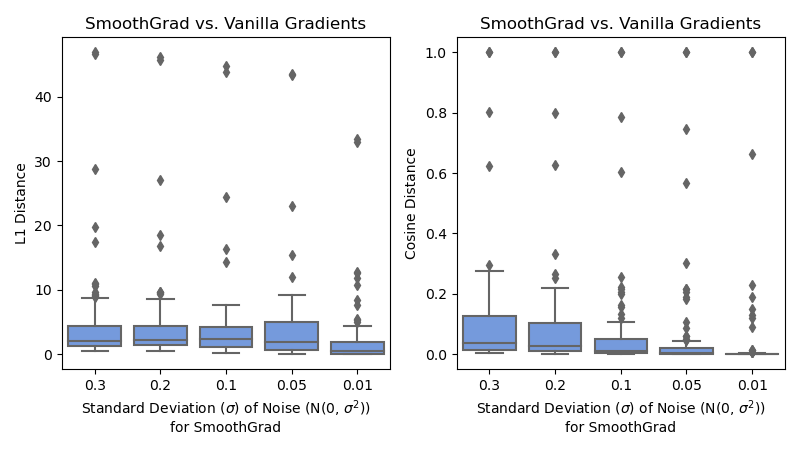}\vfill
    \includegraphics[width=0.6\textwidth]{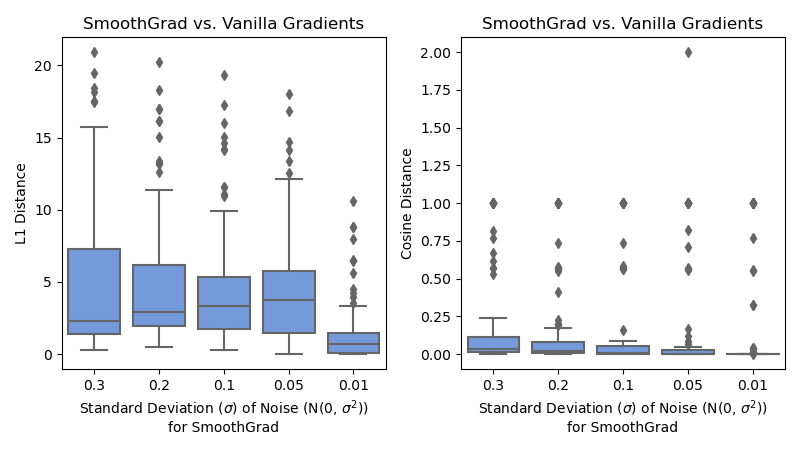}
    \caption{Using the LFA framework, explanations generated by SmoothGrad converge to those generated by Vanilla Gradients. Experiments performed on the HELOC dataset for logistic regression (Row 1), NNA (Row 2), NNB (Row 3), and NNC (Row 4). The similarity of pairs of explanations are measured based on L1 distance (left column) and cosine distance (right column).}
\end{figure}

\begin{figure}[H]
    \centering
    \includegraphics[width=0.6\textwidth]{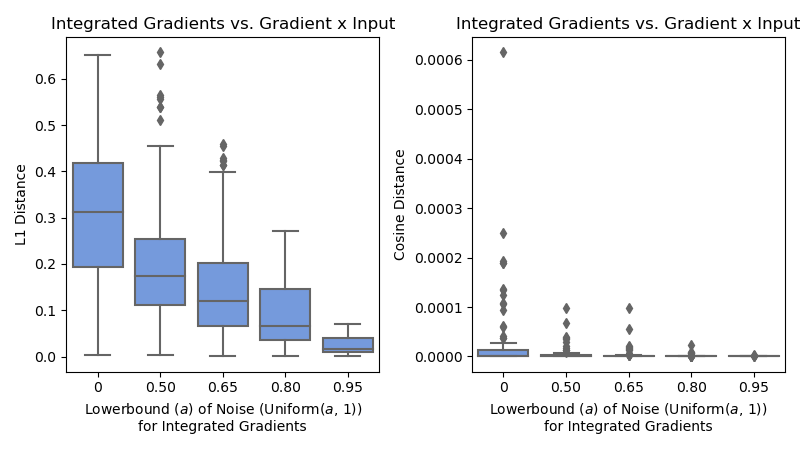}\vfill
    \includegraphics[width=0.6\textwidth]{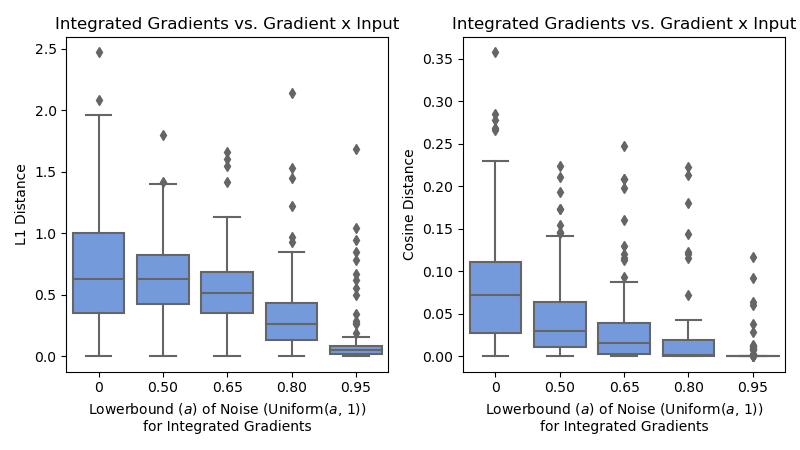}\vfill
    \includegraphics[width=0.6\textwidth]{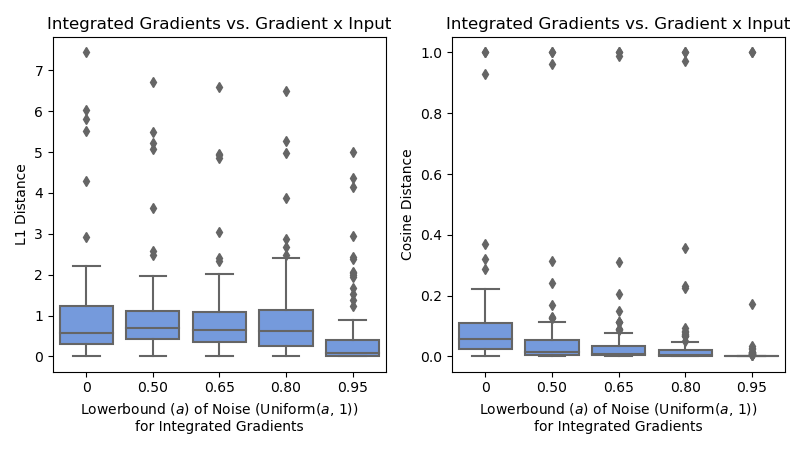}\vfill
    \includegraphics[width=0.6\textwidth]{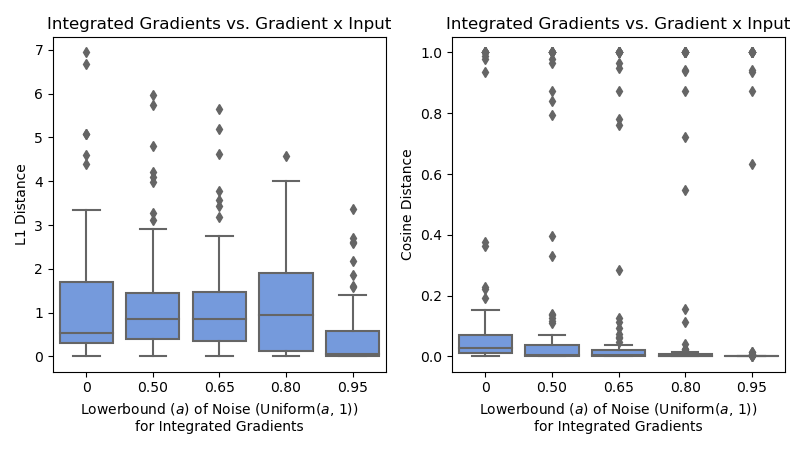}
    \caption{Using the LFA framework, explanations generated by Integrated Gradients converge to those generated by Gradient $\times$ Input. Experiments performed on the HELOC dataset for logistic regression (Row 1), NNA (Row 2), NNB (Row 3), and NNC (Row 4). The similarity of pairs of explanations are measured based on L1 distance (left column) and cosine distance (right column).}
\end{figure}

\newpage
\subsubsection{Experiment 2: $g$'s recovery of $f$}
\label{app:exp2-recovery-all}

\begin{figure}[h]
    \centering
    \includegraphics[width=0.8\textwidth]{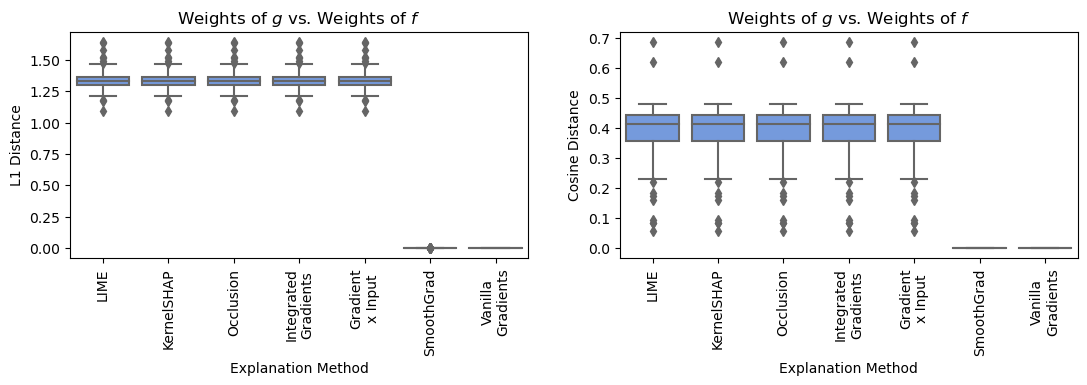}\vfill
    \includegraphics[width=0.8\textwidth]{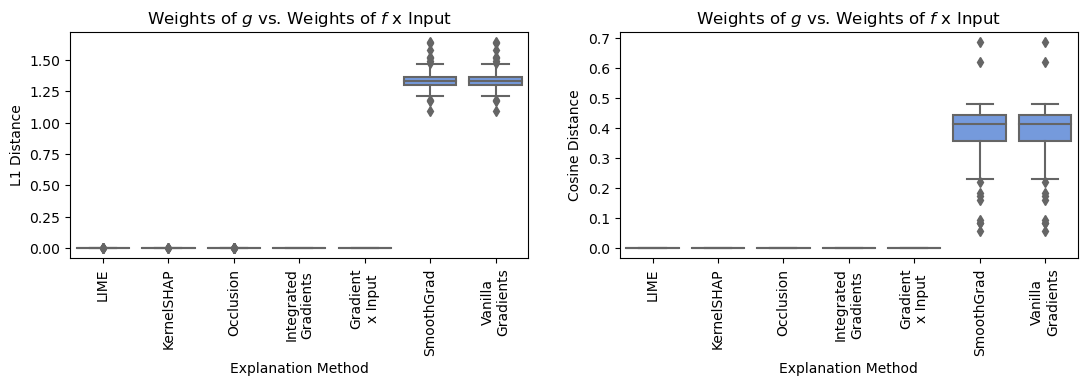}
    \caption{Analysis of $g$'s recovery of $f$ using a linear regression model trained on the WHO dataset. $g$'s weights are compared with $f$'s weights (top row) or $f$'s weights multiplied by the input (bottom row) based on L1 distance (left column) or cosine distance (right column).}
\end{figure}

\begin{figure}[h]
    \centering
    \includegraphics[width=0.8\textwidth]{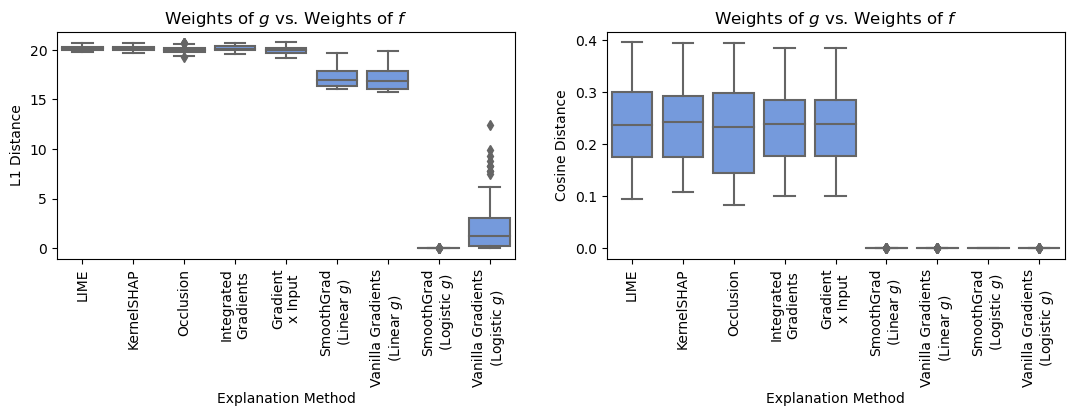}\vfill
    \includegraphics[width=0.8\textwidth]{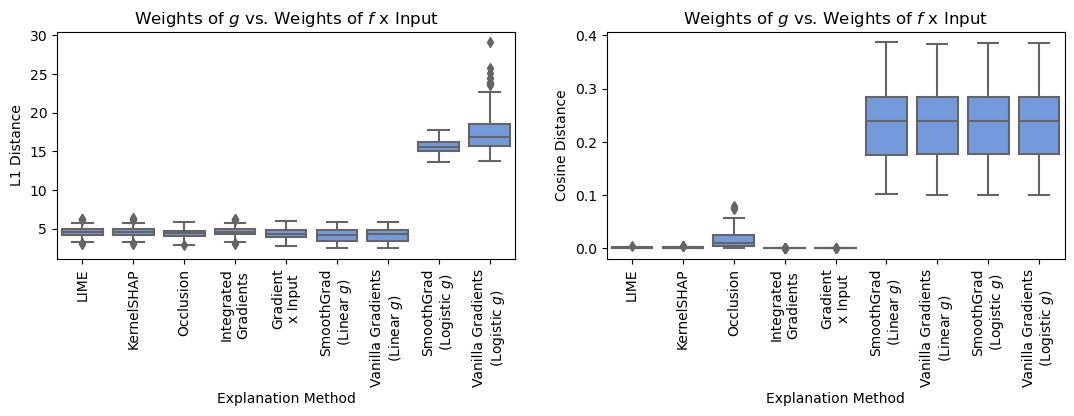}
    \caption{Analysis of $g$'s recovery of $f$ using a logistic regression model trained on the HELOC dataset. $g$'s weights are compared with $f$'s weights (top row) or $f$'s weights multiplied by the input (bottom row) based on L1 distance (left column) or cosine distance (right column).}
\end{figure}

\newpage
\subsubsection{Experiment 3: Perturbation Tests}
\label{app:exp3-p-test-all}

\begin{figure}[h]
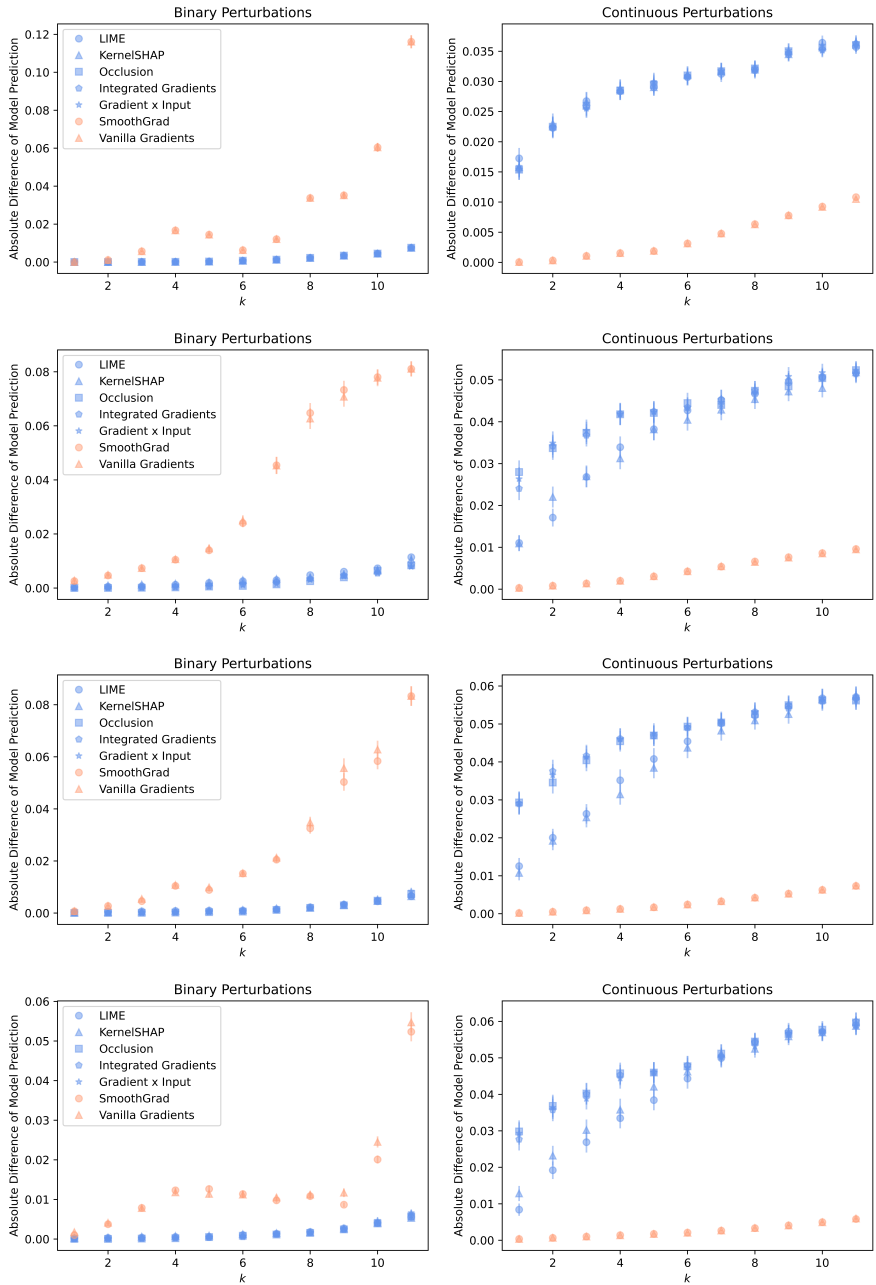

    \centering
    \includegraphics[width=\textwidth]{images/exp3-bottomk/exp3_who_linear.png}\vfill
    \includegraphics[width=\textwidth]{images/exp3-bottomk/exp3_who_ffnnA.png}\vfill
    \includegraphics[width=\textwidth]{images/exp3-bottomk/exp3_who_ffnnB.png}\vfill
    \includegraphics[width=\textwidth]{images/exp3-bottomk/exp3_who_ffnnC.png}
    \caption{Perturbation tests based on bottom-$k$ features using binary noise (left column) or continuous noise (right column) performed on the WHO dataset for linear regression (Row 1), NN1 (Row 2), NN2 (Row 3), and NN3 (Row 4). The lower the curve, the better a method identifies unimportant features. (Note: Row 2 is a duplicate of Figure~\ref{fig:exp3-ptest}).}
\end{figure}

\begin{figure}
    \centering
    \includegraphics[width=\textwidth]{images/exp3-bottomk/exp3_heloc_logistic.png}\vfill
    \includegraphics[width=\textwidth]{images/exp3-bottomk/exp3_heloc_ffnnA.png}\vfill
    \includegraphics[width=\textwidth]{images/exp3-bottomk/exp3_heloc_ffnnB.png}\vfill
    \includegraphics[width=\textwidth]{images/exp3-bottomk/exp3_heloc_ffnnC.png}
    \caption{Perturbation tests based on bottom-$k$ features using binary noise (left column) or continuous noise (right column) performed on the HELOC dataset for logistic regression (Row 1), NNA (Row 2), NNB (Row 3), and NNC (Row 4). The lower the curve, the better a method identifies unimportant features.}
\end{figure}

\begin{figure}[h]
    \centering
    \includegraphics[width=\textwidth]{images/exp3-topk/exp3_who_linear.png}\vfill
    \includegraphics[width=\textwidth]{images/exp3-topk/exp3_who_ffnnA.png}\vfill
    \includegraphics[width=\textwidth]{images/exp3-topk/exp3_who_ffnnB.png}\vfill
    \includegraphics[width=\textwidth]{images/exp3-topk/exp3_who_ffnnC.png}
    \caption{Perturbation tests based on top-$k$ features using binary noise (left column) or continuous noise (right column) performed on the WHO dataset for linear regression (Row 1), NN1 (Row 2), NN2 (Row 3), and NN3 (Row 4). The higher the curve, the better a method identifies important features.}
\end{figure}

\begin{figure}
    \centering
    \includegraphics[width=\textwidth]{images/exp3-topk/exp3_heloc_logistic.png}\vfill
    \includegraphics[width=\textwidth]{images/exp3-topk/exp3_heloc_ffnnA.png}\vfill
    \includegraphics[width=\textwidth]{images/exp3-topk/exp3_heloc_ffnnB.png}\vfill
    \includegraphics[width=\textwidth]{images/exp3-topk/exp3_heloc_ffnnC.png}
    \caption{Perturbation tests based on top-$k$ features using binary noise (left column) or continuous noise (right column) performed on the HELOC dataset for logistic regression (Row 1), NNA (Row 2), NNB (Row 3), and NNC (Row 4). The higher the curve, the better a method identifies important features.}
\end{figure}

\end{document}